\documentclass{article}

\usepackage[margin=1in]{geometry}
\usepackage[numbers]{natbib}

\usepackage[utf8]{inputenc} 
\usepackage[T1]{fontenc}    
\usepackage[colorlinks,linkcolor=blue,filecolor=blue,citecolor=black,urlcolor=blue]{hyperref}       
\usepackage{url}    

\usepackage{booktabs}      
\usepackage{amsfonts}       
\usepackage{nicefrac}       
\usepackage{microtype}      
\usepackage[table]{xcolor}         
\usepackage{tocvsec2}
\usepackage{titletoc}

\usepackage{xspace}

\usepackage{amsmath}
\usepackage{amsthm}
\usepackage{amssymb}

\usepackage{multicol}
\usepackage{multirow}
\usepackage{makecell}
\usepackage{enumitem}
\usepackage{wrapfig}
\usepackage{dsfont}

\usepackage{graphicx,epstopdf}
\usepackage{subfigure} 

\usepackage{algorithm}
\usepackage{algpseudocode}

\usepackage{threeparttable, booktabs}
\usepackage{mathtools}
\mathtoolsset{showonlyrefs=false}
\usepackage{float}

\usepackage{hyperref}[pdfusetitle]
\hypersetup{colorlinks,urlcolor=blue,hypertexnames=false}
\usepackage{xcolor,pifont}
\usepackage{MnSymbol,wasysym}

\usepackage{tabularray}

\usepackage[most]{tcolorbox}

\newtcblisting{tikzexample}{
    sidebyside,
    center lower,
    bicolor,
    colbacklower=white,
    sharp corners,
    frame engine=empty
}

\usepackage{quiver}

\restylefloat{algorithm} 
\newtheorem{theorem}{Theorem}

\newtheorem{lemma}{Lemma}

\usepackage{diagbox}

\usepackage{bm,adjustbox,array,tabularx,flafter,placeins}

\newcommand{\eeqref}[1]{(\ref{#1})}
\providecommand{\ropeResult}[3]{\makecell[c]{#1\\[2pt]\(#2\)\\[2pt]\(#3\)}}

\newtcolorbox{limbox}{
  enhanced,
  breakable,
  colback=orange!8,
  frame hidden,
  boxrule=0pt,
  borderline west={1.2pt}{0pt}{orange!55!black},
  sharp corners,
  left=6pt,
  right=2pt,
  top=3pt,
  bottom=3pt,
  before skip=5pt,
  after skip=5pt
}

\newtcolorbox{thmboxgreen}{
  enhanced, breakable, colback=green!5,
  frame hidden, boxrule=0pt,
  borderline west={1.2pt}{0pt}{green!55!black},
  sharp corners, left=6pt, right=2pt, top=3pt, bottom=3pt,
  before skip=5pt, after skip=5pt
}

\newtcolorbox{thmboxred}{
  enhanced, breakable, colback=red!5,
  frame hidden, boxrule=0pt,
  borderline west={1.2pt}{0pt}{red!55!black},
  sharp corners, left=6pt, right=2pt, top=3pt, bottom=3pt,
  before skip=5pt, after skip=5pt
}

\title{Where Position Matters: Profiling-Guided Layer-Wise RoPE Allocation in Transformers}

\author{%
  \normalfont
  \begin{minipage}[t]{0.96\textwidth}
    \centering
    Zhiwei Liu$^{*}$,\quad Boao Kong$^{*,\dagger}$,\quad Bing Liu
  \end{minipage}%
}

\begin{document}

\setlength{\parindent}{0pt}
\setlength{\parskip}{0.5em}
\maketitle

% Author notes appear at the foot of the title page, not beneath the author list.
\begingroup
\renewcommand{\thefootnote}{\fnsymbol{footnote}}
\footnotetext[1]{Zhiwei Liu and Boao Kong contributed equally to this work.}
\footnotetext[2]{Boao Kong is the corresponding author. Email: \texttt{kongboao@stu.pku.edu.cn}.}
\endgroup

\begin{abstract}
Rotary Position Embedding (RoPE) typically uses uniform coverage across Transformer layers despite heterogeneous positional requirements. Phase-off interventions reveal non-uniform, model-specific rotary-phase dependence, while coverage interventions relate it to coverage-upgrade benefits. We propose profiling-guided layer-wise RoPE allocation, assigning high/low coverage under a fixed budget using early coverage-benefit profiles. The framework supports fixed-width query/key representations and physical rotary-dimension allocation in Multi-head Latent Attention (MLA). Normalized-depth mapping transfers profiles within model families, enabling small-model measurements to guide larger-model allocations before pretraining. A brief analysis relates selection and transfer to score calibration and inter-layer interactions. From-scratch experiments demonstrate validation-loss improvements over available budget-matched baselines, while analytical resource accounting identifies favorable quality--storage--compute trade-offs for physical MLA. These results connect positional demand to transferable, efficient architecture design.
\end{abstract}

\section{Introduction}
Transformers~\citep{vaswani2017attention,dai2019transformer} model interactions within a sequence through attention, with positional encodings~\citep{shaw2018self,press2021train} providing information about token order and distance. Rotary Position Embedding (RoPE)~\citep{su2024roformer} applies position-dependent rotations to query and key dimension pairs, allowing attention scores to encode relative positional relationships. This mechanism introduces two architectural choices: \textbf{how many dimensions receive rotary transformations at each layer, and how these dimensions are distributed across network depth.} Connecting these choices to the positional requirements of individual layers creates a design space for improving language modeling under resource constraints.

Recent studies demonstrate the flexibility of rotary configurations. A systematic investigation of partial RoPE~\citep{khan2026fractional} examines how rotary coverage affects training dynamics and convergence, showing that reduced coverage can achieve performance comparable to full rotation. RNoPE-SWA~\citep{yang2026rope} interleaves RoPE-based sliding-window attention with global attention using no positional embeddings, exploiting complementary layer functions to improve quality and efficiency. RFIS and RPD~\citep{shi2026modern} further characterize attention-head specialization through frequency interventions and changes in attention distributions, providing architectural guidance for hybrid attention. These advances highlight the flexibility of rotary design and motivate adapting coverage to the positional demands of individual layers. Turning this idea into a practical design procedure requires reliable demand measurements, explicit budgeted allocation rules, and cost-effective access to profiling information across model scales, bringing the following challenges into focus.

\textbf{Limitations of existing approaches.} The practical value of a positional profile depends on its connection to task performance, its ability to guide layer-wise allocation, and the cost of obtaining it at the target scale.

\textbf{L1. Limited calibration between positional dependence and task-level utility.} Frequency-level analyses~\citep{barbero2024round} and head-level functional profiling~\citep{shi2026modern} reveal how rotary modulation shapes attention behavior. Translating these observations into layer-wise allocation requires a quantitative connection to the language-modeling objective: how strongly does each layer depend on rotary phase, and how well does this dependence predict the benefit of increasing its coverage? Establishing this connection is essential for using positional profiles as allocation signals.

\textbf{L2. Coarse-grained allocation of rotary resources across layers.} Uniform partial RoPE~\citep{khan2026fractional} assigns a common coverage ratio across depth, while RNoPE-SWA~\citep{yang2026rope} uses a predefined interleaving structure. These allocation rules offer limited flexibility for matching coverage to model-specific layer requirements. Varying coverage-upgrade benefits call for explicit allocation and matched-budget layer-placement comparisons.

\textbf{L3. Repeated profiling costs across model scales.} Checkpoint-based positional analyses~\citep{barbero2024round,shi2026modern} require trained model states and additional evaluations. Applying this workflow separately at each target scale incurs repeated costs for checkpoint acquisition, interventions, and configuration selection. Reusing structural information from smaller models therefore becomes an important step toward determining larger-model allocations before training.

These challenges motivate a unified framework for rotary allocation with efficient reuse of profiling information. We therefore ask:
\begin{thmboxgreen}
\textit{How can positional demand guide \textbf{model-specific RoPE allocation} under a fixed budget, and how can \textbf{small-model profiles} support larger-scale architecture design?}
\end{thmboxgreen}

\textbf{Main contributions.} We introduce \textbf{profiling-guided layer-wise RoPE allocation}, connecting positional-demand measurement, budgeted architecture design, and cross-scale profile transfer. We first quantify layer-wise dependence through phase-off interventions and calibrate the resulting scores against coverage-upgrade benefits. We then assign high and low rotary coverage using early coverage-benefit scores. Finally, a normalized-depth mapping transfers small-model profiles into allocations specified before target-model training. Our contributions are threefold.

\textbf{C1. Task-level positional-demand measurement and calibration (addressing L1).} We use \emph{phase-off profiling}, which replaces a layer's rotary transformation with the identity while preserving its pre-rotation query/key representations and attention mask. Paired changes in validation loss quantify the layer's dependence on rotary phase. Across four pretrained models and 136 layers, we identify strongly non-uniform and model-specific demand profiles. Controlled coverage interventions on selected high-, medium-, and low-demand layers further reveal strong ranking agreement between phase-off scores and coverage-upgrade benefits, connecting positional dependence to actionable architectural changes.

\textbf{C2. Model-specific RoPE allocation under a fixed budget (addressing L2).} We assign higher coverage to layers with larger early coverage-benefit scores and lower coverage to the remaining layers, subject to a prescribed total rotary-pair budget. Within each allocation experiment, configurations share a master frequency bank and use nested frequency-pair subsets. We provide a fixed-carrier implementation that preserves complete query/key widths, together with a physical rotary-dimension implementation for Multi-head Latent Attention (MLA)~\citep{liu2024deepseekv2,liu2024deepseek}. The latter connects layer-wise allocation to projection parameters, attention computation, and positional-key storage. A brief motivating analysis explains how upgrade benefits, score calibration, and inter-layer interactions inform the selection rule.

\textbf{C3. Within-family cross-scale allocation transfer (addressing L3).} We map small-model allocation profiles to larger models along normalized depth and combine the transferred scores with the target budget to determine layer-wise coverage before training. This procedure reuses source-model measurements for architecture selection at larger scales. From-scratch experiments with transferred allocations examine their language-modeling behavior under different attention parameterizations, connecting source-side structural measurements to target-scale training outcomes.

\subsection{Related Work}
\label{sec:related_work}

\textbf{Positional encoding and rotary coverage.} YaRN~\citep{peng2024yarn} combines frequency-dependent interpolation with attention scaling to extend context windows, while LongRoPE~\citep{ding2024longrope} uses non-uniform positional interpolation and progressive extension to support longer inputs. GPT-NeoX-20B~\citep{black2022gptneox} and Pythia~\citep{biderman2023pythia} apply rotary transformations to a subset of query/key dimensions, treating coverage as an architectural hyperparameter. The systematic study of partial RoPE~\citep{khan2026fractional} further examines the effects of this fraction on convergence and training stability. Our work develops coverage into a layer-wise decision variable and studies its allocation under a prescribed aggregate budget.

\textbf{Positional dependence and functional specialization.} Barbero et al.~\citep{barbero2024round} analyze how rotary frequencies support positional attention patterns and discuss the potential role of low-frequency channels in carrying semantic information. RNoPE-SWA~\citep{yang2026rope} combines local positional modeling and global retrieval through a hybrid layer configuration. RFIS and RPD~\citep{shi2026modern} use frequency interventions to characterize retrieval and positional heads, motivating layer-specific allocations in hybrid attention architectures. Our approach uses language-modeling loss to connect layer-wise phase dependence with coverage-upgrade benefits and budgeted rotary allocation.

\textbf{Structure-aware efficient architectures.} Multi-query attention (MQA)~\citep{shazeer2019fast,chowdhery2023palm} reduces decoding memory traffic by sharing key/value heads, while grouped-query attention (GQA)~\citep{ainslie2023gqa,touvron2023llama,grattafiori2024llama} introduces grouped sharing to balance model quality and inference efficiency. MLA~\citep{liu2024deepseekv2,liu2024deepseek} compresses the key/value cache through a low-rank latent representation, with a decoupled rotary branch providing a direct interface for positional-dimension allocation. Low-rank modeling approaches~\citep{yuan2023asvd,kong2026cr,wang2025svd,wang2025svdv2} exploit activation distributions or cross-layer correlations to construct compact parameterizations and reduce parameter and computational costs. For structured pruning, SliceGPT~\citep{ashkboos2024slicegpt} combines computationally invariant rotations with dimension deletion to reduce network width, while ShortGPT~\citep{men2025shortgpt} identifies redundant layers through changes between layer inputs and outputs. Our work uses layer-wise positional demand to organize rotary resources and transfers profiles within model families to determine target architectures before training.

\section{Preliminaries and Problem Formulation}
\label{sec:preliminaries}

We introduce attention, rotary position embeddings, and MLA, followed by the layer-wise rotary allocation problem.

\textbf{Attention and RoPE.} For an \(L\)-layer Transformer, let \(X_\ell\in\mathbb{R}^{T\times d}\) contain \(T\) tokens of hidden width \(d\). Suppressing head indices, attention~\citep{vaswani2017attention} uses \(Q_\ell=X_\ell W_\ell^Q\), \(K_\ell=X_\ell W_\ell^K\), and \(V_\ell=X_\ell W_\ell^V\). RoPE~\citep{su2024roformer} sets \(\widetilde q_{\ell,t}=R_{\ell,t}q_{\ell,t}\) and \(\widetilde k_{\ell,t}=R_{\ell,t}k_{\ell,t}\), rotating selected pairs \(j\) by \(t\theta_j\) and applying identity elsewhere. Stacking transformed vectors gives
\begin{equation}
\operatorname{Attn}_\ell(X_\ell)=\operatorname{Softmax}\!\left(\frac{\widetilde Q_\ell\widetilde K_\ell^\top}{\sqrt{d_h}}+\mathcal A\right)V_\ell,
\label{eq:attention}
\end{equation}
where \(d_h\) is the per-head query/key width and \(\mathcal A\) is the causal mask. All dimensions participate in matching.

\textbf{Multi-head latent attention.} MLA~\citep{liu2024deepseekv2,liu2024deepseek} uses a shared latent \(C_\ell^{KV}=X_\ell W_\ell^{DKV}\in\mathbb{R}^{T\times d_c}\), with \(K_\ell^C=C_\ell^{KV}W_\ell^{UK}\) and \(V_\ell=C_\ell^{KV}W_\ell^{UV}\). Content and rotary branches form
\begin{equation}
\begin{aligned}
\widetilde Q_\ell&=\bigl[Q_\ell^C,\operatorname{RoPE}(Q_\ell^R)\bigr],\quad \widetilde K_\ell=\bigl[K_\ell^C,\operatorname{RoPE}(K_\ell^R)\bigr].
\end{aligned}
\label{eq:mla_query_key}
\end{equation}
Caching the KV latent and cross-head shared positional key requires \(d_c+d_{\mathrm{rope}}^{(\ell)}\) logical elements per layer and token, where \(d_{\mathrm{rope}}^{(\ell)}\) is the per-head rotary width.

\textbf{Rotary coverage and layer-wise budget.} Let \(M_\ell\) and \(m_\ell\) denote candidate and active rotary pairs per query/key head, respectively, with coverage shared across heads:
\begin{equation}
\begin{gathered}
\rho_\ell=\frac{m_\ell}{M_\ell},\quad
m_\ell\in\{0,\ldots,M_\ell\},\quad \sum_{\ell=1}^{L}m_\ell=B.
\end{gathered}
\label{eq:rotary_budget}
\end{equation}
Uniform coverage \(\rho_\ell=\rho_0\) gives \(B=\rho_0\sum_\ell M_\ell\). We seek an allocation \(\boldsymbol m=(m_1,\ldots,m_L)\) under budget \(B\), evaluated by validation loss. For architecture allocation, coverage levels use nested pair subsets from a shared frequency bank. Fixed-carrier allocation preserves complete query/key widths; physical MLA sets \(d_{\mathrm{rope}}^{(\ell)}=2m_\ell\), keeping content-branch and latent widths fixed.

\section{Layer-Wise Positional Demand and Coverage-Upgrade Benefits}
\label{sec:positional_demand}

Section~\ref{sec:preliminaries} expresses the RoPE configuration as a layer-wise allocation vector \(\boldsymbol m\) satisfying Equation~\eeqref{eq:rotary_budget}. To determine rotary coverage at individual layers, we perform layer-wise phase interventions on pretrained models, measure the resulting changes in task loss, and examine their relationship with coverage-upgrade benefits.

\subsection{Experimental Setup and Positional-Demand Measurement}
\label{sec:phaseoff_measurement}

\textbf{Experimental setup.} We conduct layer-wise evaluations on DeepSeek-MoE-16B~\citep{dai2024deepseekmoe}, LLaMA-2-7B~\citep{touvron2023llama}, LLaMA-3.2-3B~\citep{meta2024llama32}, and Qwen3-30B~\citep{yang2025qwen3}. All models are evaluated on the same 500 text samples, with model weights frozen throughout measurement. Layer numbers in this analysis are one-based. Within each model, the reference and intervention configurations use identical input token sequences, attention masks, and loss-evaluation positions. Paired loss differences quantify the effect of changing the rotary configuration.

\textbf{Phase-off score.} Following the fixed-carrier parameterization in Section~\ref{sec:preliminaries}, we use the configuration with all candidate rotary pairs activated as the reference. Let \(\mathcal L_\ell(x;\rho)\) denote the average per-token validation loss on evaluation text sample \(x\) when layer \(\ell\) uses coverage \(\rho\), with all other layers retaining their reference configurations. Phase-off corresponds to \(\rho=0\), setting \(R_{\ell,t}=I\) while preserving the pre-rotation query/key representations and the value branch. Given an evaluation sample set \(\mathcal D_{\mathrm{prof}}\), we define the layer-wise positional-demand score as
\begin{equation}
\widehat D_\ell
=
\frac{1}{|\mathcal D_{\mathrm{prof}}|}
\sum_{x\in\mathcal D_{\mathrm{prof}}}
\left[
\mathcal L_\ell(x;0)
-
\mathcal L_\ell(x;1)
\right].
\label{eq:phaseoff_score}
\end{equation}
A larger \(\widehat D_\ell\) indicates a greater loss increase when the layer's rotary phase is disabled. Evaluating each layer yields the positional-demand profile \(\widehat{\boldsymbol D}=(\widehat D_1,\ldots,\widehat D_L)\), with confidence intervals computed from sample-level paired loss differences.

\subsection{Empirical Distribution of Layer-Wise Positional Demand}
\label{sec:demand_distribution}

\textbf{Variation across layers.} Figure~\ref{fig:layer_profiles} displays signed demand proportional to \(\widehat D_\ell/\max_j\widehat D_j\) at normalized depth \((\ell-\tfrac12)/L\). Each profile exhibits pronounced local peaks, with substantial differences between high-scoring layers and their neighbors. To quantify concentration, we measure the positive-demand share of a depth region as the sum of \([\widehat D_\ell]_+\) within that region divided by the corresponding sum over all layers, where \([a]_+=\max(a,0)\). Table~\ref{tab:demand_concentration} shows that positive demand in DeepSeek-MoE-16B, LLaMA-2-7B, and LLaMA-3.2-3B is concentrated in the first quarter of depth. In LLaMA-2-7B, layer 2 alone accounts for \(40.73\%\) of the model's total positive demand.

\textbf{Depth distributions across models.} Figure~\ref{fig:layer_profiles} also reveals differences between models. The principal peaks in the DeepSeek and LLaMA models occur in early layers, whereas Qwen3-30B exhibits elevated demand at both ends of the network. Accordingly, its final quarter of depth accounts for a substantially larger share of positive demand than in the other models. The locations of high-scoring layers and the concentration of demand together characterize each model's layer-wise positional structure.

\begin{figure}[t!]
\centering
\includegraphics[width=0.9\linewidth]{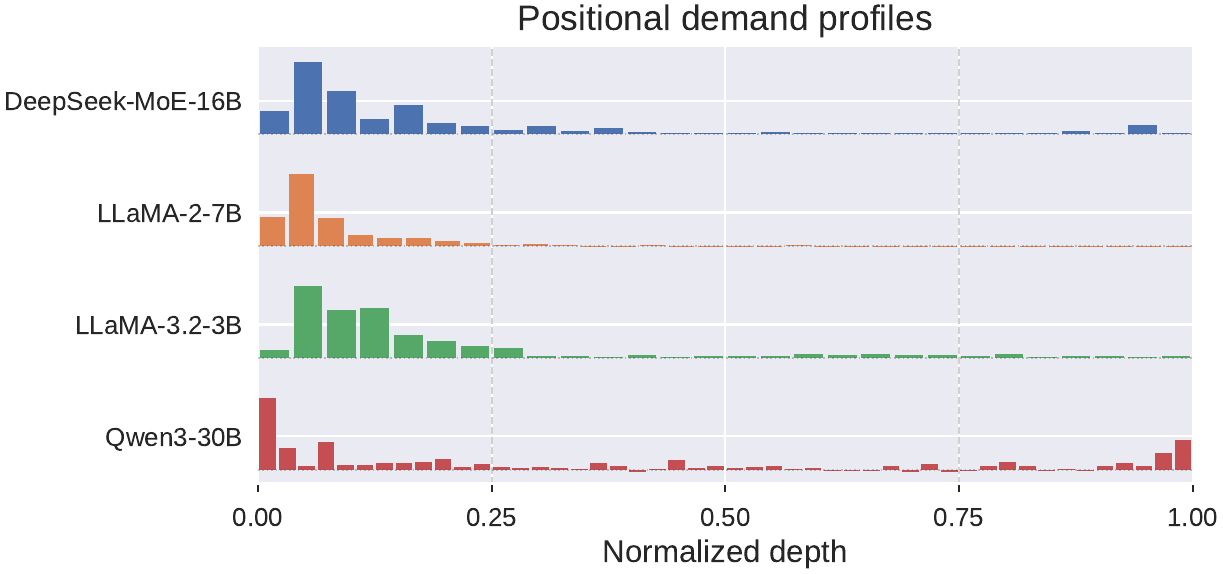}
\caption{Layer-wise phase-off demand across normalized depth. Signed scores are scaled by each model's positive peak for display.}
\label{fig:layer_profiles}
\end{figure}

\begin{table}[t!]
\centering
\small
\setlength{\tabcolsep}{3pt}
\renewcommand{\arraystretch}{1.14}
\caption{Positive phase-off demand shares in the first and last depth quarters of each model.}
\label{tab:demand_concentration}
\begin{threeparttable}
\begin{adjustbox}{max width=\columnwidth}
\begin{tabular}{@{}lcc@{}}
\toprule
Model
& First quarter
& Last quarter \\
\midrule
DeepSeek-MoE-16B & 79.57\% & 6.74\% \\
LLaMA-2-7B      & 93.79\% & 0.75\% \\
LLaMA-3.2-3B    & 79.16\% & 5.74\% \\
Qwen3-30B        & 55.42\% & 24.57\% \\
\bottomrule
\end{tabular}
\end{adjustbox}
\end{threeparttable}
\end{table}

\subsection{Relating Positional Demand to Coverage-Upgrade Benefits}
\label{sec:coverage_calibration}

\textbf{Coverage intervention setup.} To examine how positional-demand scores inform coverage allocation, we select two layers from each model's high-, medium-, and low-demand groups and measure task loss at different coverage levels on 100 calibration samples. Configurations retain the pretrained frequency bank and select rotary pairs with a fixed stratified policy. Unrotated dimensions receive the identity transformation and continue to participate in query/key matching. All coverage configurations within calibration use the same evaluation text samples, with other layers retaining their reference configurations. The profiling and calibration sample sets are denoted separately by \(\mathcal D_{\mathrm{prof}}\) and \(\mathcal D_{\mathrm{cal}}\). We evaluate \(\rho\in\{0,0.25,0.5,0.75,1\}\) and consider three upgrades: \(25\%\!\rightarrow75\%\), \(50\%\!\rightarrow100\%\), and \(25\%\!\rightarrow100\%\).

For low and high coverage levels \(\rho_L<\rho_H\), let \(\mathcal D_{\mathrm{cal}}\) denote the fixed calibration sample set. We define the coverage-upgrade benefit at layer \(\ell\) as
\begin{equation}
\begin{aligned}
&\widehat G_\ell(\rho_L,\rho_H)=\frac{1}{|\mathcal D_{\mathrm{cal}}|}
\sum_{x\in\mathcal D_{\mathrm{cal}}}
\left[\mathcal L_\ell(x;\rho_L)-\mathcal L_\ell(x;\rho_H)\right].
\end{aligned}
\label{eq:coverage_upgrade_gain}
\end{equation}
A larger \(\widehat G_\ell\) indicates a greater reduction in task loss from increasing the layer's rotary coverage.

\textbf{Correspondence between demand and benefit.} Figure~\ref{fig:coverage_calibration} relates phase-off scores to the benefits of upgrading coverage from \(50\%\) to \(100\%\). Within each model, high-demand layers typically exhibit larger upgrade benefits, while low-demand layers yield smaller benefits. For the \(50\%\!\rightarrow100\%\) upgrade, within-model Spearman correlations are \(0.943,1.000,0.943,1.000\) in the model order above; the other two intervals range from \(0.943\) to \(1.000\). These results connect the loss changes induced by phase-off interventions to the utility of increasing rotary coverage.

\textbf{Coverage responses across layers.} The benefits of additional coverage arise at different intervals even within one model. In LLaMA-2-7B, layer 1 recovers most positional utility by half coverage, whereas layer 2 improves mainly between \(50\%\) and \(75\%\) coverage. Appendix Figure~\ref{fig:coverage_responses} compares these responses under the same evaluation samples.

\begin{figure}[t!]
\centering
\includegraphics[width=0.9\linewidth]{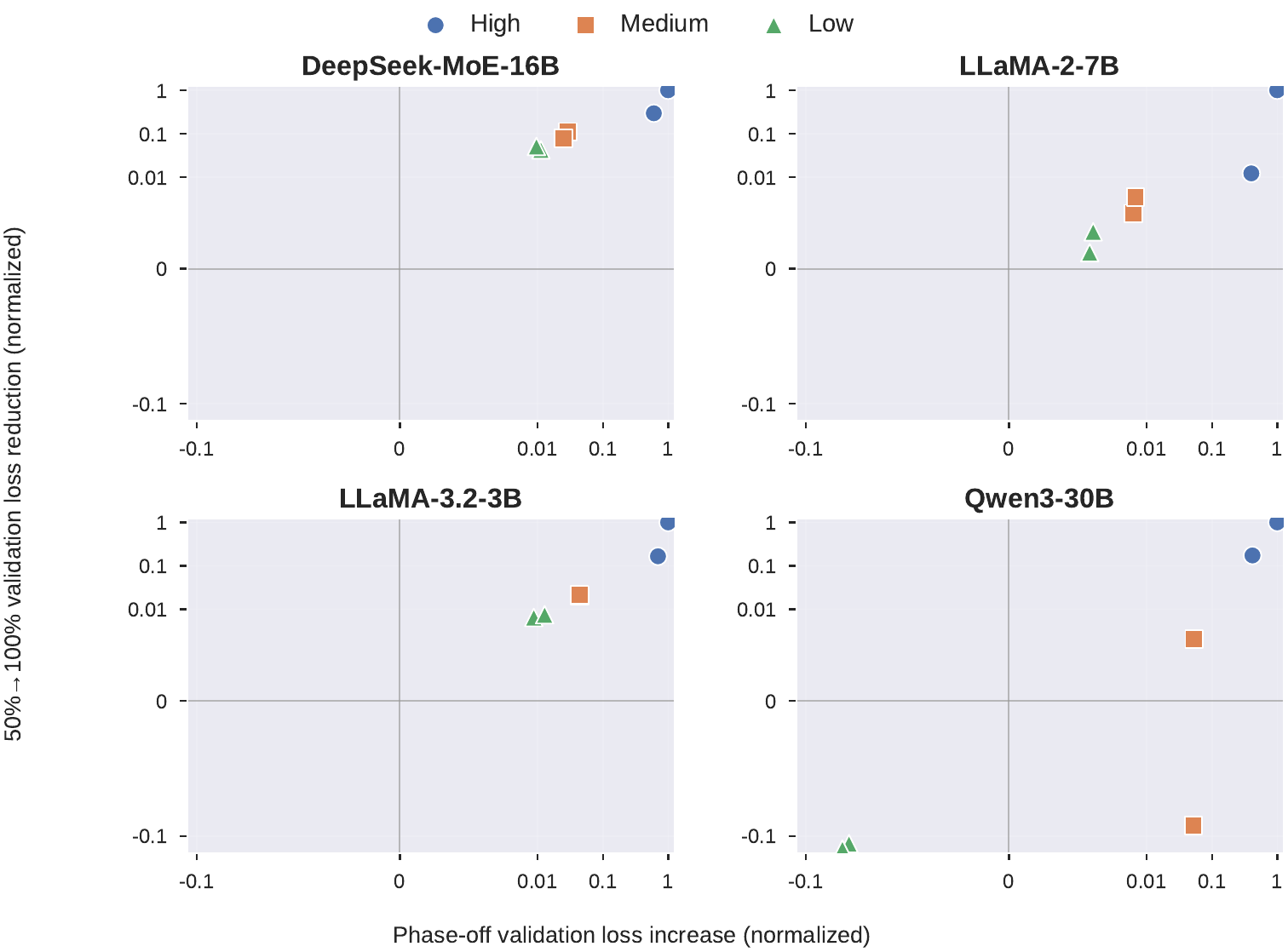}
\caption{Phase-off demand versus validation-loss gain from \(50\%\!\rightarrow100\%\) coverage. Six layers per model; signed axes use within-model peak normalization.}
\label{fig:coverage_calibration}
\end{figure}

Combining the layer-wise distributions with the coverage intervention results yields the following observation:
\begin{limbox}
\textbf{Key observation.} Positional demand in pretrained models is strongly non-uniform across depth, with model-specific high-demand regions. Across the evaluated calibration layers and coverage upgrades, phase-off scores exhibit strong ranking agreement with the measured upgrade benefits, providing a task-level basis for selecting high-coverage layers.
\end{limbox}

This observation motivates demand-guided budget allocation: under Equation~\eeqref{eq:rotary_budget}, the profile guides the selection of high-coverage layers, and high/low configurations organize rotary resources across depth. Section~\ref{sec:motivating_analysis} analyzes this design principle, and Section~\ref{sec:rope_allocation} develops the allocation method and its cross-scale implementation.

\section{Motivating Analysis}
\label{sec:motivating_analysis}

We analyze budgeted coverage allocation at a frozen checkpoint to motivate the layer-selection and profile-transfer rules. Consider high/low configurations with a common per-layer rotary-pair upgrade cost, so that Equation~\eeqref{eq:rotary_budget} permits \(K\) high-coverage layers. Let \(F(S)\) be the loss of high-coverage set \(S\), using identity rotation for unselected Dense pairs and zeroing unselected MLA positional coordinates, as specified in Appendix~\ref{app:allocation_setup}. Define \(G_\ell=F([L]\setminus\{\ell\})-F([L])\). Let \(s_\ell=[a_\ell]_+\), or a common positive normalization of these priorities, where \(a_\ell\) is the recorded matched-marginal allocation score obtained from Equation~\eeqref{eq:matched_marginal_score}. Denote its top-\(K\) set by \(\widehat S\) and a loss-minimizing \(K\)-layer set by \(S^\star\).

\begin{thmboxred}
\begin{theorem}[Informal]
\textbf{Calibrated profiles guide budgeted allocation.}
\label{thm:calibrated_allocation_main}
Suppose \(|G_\ell-\alpha s_\ell|\leq\varepsilon\) for every layer, with \(\alpha>0\), and joint losses satisfy \(F(S)=C-\sum_{\ell\in S}G_\ell+I(S)\), where \(C\) is independent of \(S\) and \(|I(S)|\leq\delta\) for every feasible set. Then
\begin{equation}
0\leq F(\widehat S)-F(S^\star)\leq 2K\varepsilon+2\delta.
\label{eq:allocation_regret_main}
\end{equation}
Exact calibration and additive effects give an optimal high/low allocation. The formal statement and proof appear in Appendix~\ref{app:calibrated_allocation}.
\end{theorem}
\end{thmboxred}

Theorem~\ref{thm:calibrated_allocation_main} relates the quality of score-based layer selection to calibration accuracy and joint-configuration effects. The calibration condition connects the measured allocation priorities to the all-high reference gains.

\begin{thmboxred}
\begin{theorem}[Informal]
\textbf{Profile accuracy controls cross-scale allocation.}
\label{thm:profile_transfer_main}
Suppose a normalized target profile \(\boldsymbol s^{(t)}\) satisfies the preceding calibration and interaction conditions with parameters \(\alpha_t,\varepsilon_t,\delta_t\). If a transferred profile \(\widetilde{\boldsymbol s}\) satisfies \(\|\widetilde{\boldsymbol s}-\boldsymbol s^{(t)}\|_\infty\leq\eta\), its top-\(K_t\) set \(\widehat S_{s\to t}\) obeys
\begin{equation}
F_t(\widehat S_{s\to t})-F_t(S_t^\star)
\leq 2K_t(\varepsilon_t+\alpha_t\eta)+2\delta_t.
\label{eq:transfer_regret_main}
\end{equation}
The formal statement and proof appear in Appendix~\ref{app:profile_transfer}.
\end{theorem}
\end{thmboxred}

Theorem~\ref{thm:profile_transfer_main} quantifies how profile discrepancy propagates to target-layer selection.

\textbf{Connection to the empirical observations.} Theorem~\ref{thm:calibrated_allocation_main} formalizes the score-based selection principle motivated by the layer-wise heterogeneity and demand--benefit ranking agreement observed in Section~\ref{sec:positional_demand}. Under its conditions, for any feasible set \(S\), replacing a selected layer \(j\in S\) with an unselected layer \(i\notin S\) lowers the configuration loss whenever \(\alpha(s_i-s_j)>2\varepsilon+2\delta\), as established in Lemma~\ref{lem:demand_guided_exchange}. Thus, sufficiently separated allocation scores identify beneficial budget-preserving exchanges, giving the observed high-demand regions a direct role in layer selection. Theorem~\ref{thm:profile_transfer_main} further connects normalized-depth profile agreement to the additional allocation-error term \(2K_t\alpha_t\eta\). Together, these conclusions motivate selecting high-coverage layers from model-specific allocation profiles and reusing small-model profiles for target-scale architecture design in Section~\ref{sec:rope_allocation}.

\section{Profiling-Guided Layer-Wise RoPE Allocation}
\label{sec:rope_allocation}

Section~\ref{sec:positional_demand} connects layer-wise positional demand to coverage-upgrade benefits, and Section~\ref{sec:motivating_analysis} motivates score-based selection and cross-scale transfer. We turn these insights into a layer-wise architecture: the rotary budget determines the number of high-coverage layers, the profile determines their locations, and a smaller model from the same family supplies the allocation before target training.

\subsection{Demand-Guided Budget Allocation}
\label{sec:budgeted_selection}

We train a full-rotation reference model for \(10\%\) of its total training steps and obtain layer-wise allocation scores. Denote the measured score by \(a_\ell\) and use \(s_\ell=[a_\ell]_+\) as the upgrade priority. The phase interventions and coverage calibration in Section~\ref{sec:positional_demand} connect positional demand to configuration utility; the coverage-benefit measurements and score records used in the experiments are described in Appendix~\ref{app:exp_setup}.

Following the high/low configuration in Section~\ref{sec:motivating_analysis}, let \(m_{\ell,H}=M_\ell\) and \(m_{\ell,L}=M_\ell-\Delta m\geq0\), where \(\Delta m>0\) is a common integer upgrade increment. Given total budget \(B\), the high-coverage count and allocation are
\begin{equation}
\begin{aligned}
K&=\frac{B-\sum_{\ell=1}^{L}m_{\ell,L}}{\Delta m},\quad\widehat S=\operatorname{TopK}_{\ell\in[L]}(s_\ell),\quad m_\ell=m_{\ell,L}+\Delta m\,\mathbf 1\{\ell\in\widehat S\},\quad\rho_\ell=\frac{m_\ell}{M_\ell}.
\end{aligned}
\label{eq:budgeted_allocation}
\end{equation}
The coverage specifications and budget give an integer \(1\leq K<L\), with ties resolved by a fixed rule. The resulting allocation satisfies \(\sum_\ell m_\ell=B\).

\textbf{From positional demand to architecture.} The high-demand regions in Section~\ref{sec:positional_demand} identify layers with strong dependence on rotary phase, and coverage calibration further characterizes their upgrade benefits. Low coverage supplies a base budget, while the early matched-marginal scores in Equation~\eeqref{eq:matched_marginal_score} assign additional pairs to the selected layers. This makes \emph{where position matters} an executable architectural decision. Theorem~\ref{thm:calibrated_allocation_main} motivates this selection rule through score calibration and inter-layer interactions.

\subsection{Architectural Instantiations}
\label{sec:allocation_implementations}

\textbf{Shared frequencies and fixed-carrier allocation.} Let \(\mathcal J_\ell(m)\) denote the active frequency-pair set at layer \(\ell\), with \(|\mathcal J_\ell(m)|=m\) and the low-coverage set nested within the high-coverage set. Selected pairs retain their frequencies from the shared bank. The pretrained-model analysis uses stratified pair selection; the allocation experiments use the \texttt{high\_to\_low} order specified in Appendix~\ref{app:exp_setup}. Within each experiment, the frequency bank and selection policy remain fixed across coverage levels. Fixed-carrier allocation applies position-dependent rotation to the selected pairs and the identity elsewhere, with all dimensions participating in attention matching. Queries and keys use matched frequency sets, and all heads in a layer share its coverage configuration.

\textbf{Physical MLA allocation.} For MLA~\citep{liu2024deepseekv2,liu2024deepseek}, we set \(d_{\mathrm{rope}}^{(\ell)}=2m_\ell\), construct rotary-query and cross-head shared positional-key projections at this width, and use the frequencies indexed by \(\mathcal J_\ell(m_\ell)\). The content branch, latent width, and head count remain fixed, while attention scaling follows the full-rotation reference. The allocation vector therefore determines both the layer-wise distribution of positional resources and the corresponding projection structure.

\subsection{Within-Family Profile Transfer}
\label{sec:allocation_transfer}

For source and target depths \(L_s\) and \(L_t\), define normalized depths \(x_i^{(s)}=(i-\tfrac12)/L_s\) and \(x_j^{(t)}=(j-\tfrac12)/L_t\). Normalize the source scores obtained at \(10\%\) training progress by their positive-score mean and map them to target depths:
\begin{equation}
\begin{aligned}
\widehat s_i^{(s)}&=\frac{[a_i^{(s)}]_+}{\widehat\mu_s},\quad
\widehat\mu_s=\frac{1}{L_s}\sum_{u=1}^{L_s}[a_u^{(s)}]_+>0,\quad \widetilde{\boldsymbol s}=W\widehat{\boldsymbol s}^{(s)}.
\end{aligned}
\label{eq:source_profile_transfer}
\end{equation}
Following Appendix~\ref{app:profile_transfer}, \(W\in\mathbb R^{L_t\times L_s}\) uses piecewise-linear interpolation with constant endpoint extension, satisfying \(W_{ji}\geq0\) and \(\sum_iW_{ji}=1\). Substituting the transferred scores and target budget into Equation~\eeqref{eq:budgeted_allocation} gives the target allocation \(\boldsymbol m^{(t)}\). We then construct the target architecture, initialize its weights, and train from scratch with the allocation and frequency sets fixed.

\textbf{Small-model profiles for target architecture design.} Normalized depth maps source high-demand regions to different target depths, while the target budget determines their coverage levels. One early profile supports multiple target depths and rotary budgets, amortizing structural measurement across architecture designs. Theorem~\ref{thm:profile_transfer_main} connects the mapped score discrepancy to target selection quality and motivates this procedure.

\begin{algorithm}[t]
\caption{Profiling-guided layer-wise RoPE allocation}
\label{alg:rope_allocation}
\small
\begin{algorithmic}[1]
\Require Reference model and training recipe; evaluation samples; target architecture; rotary budget and feasible high/low specifications; shared frequency bank; target training data and recipe.
\Ensure Target allocation \(\boldsymbol m^{(t)}\) and pretrained target model.
\State Train the reference model for \(10\%\) of its total steps and obtain layer-wise allocation scores.
\If{using cross-scale transfer}
    \State Normalize and map the scores to target depths using Equation~\eeqref{eq:source_profile_transfer}.
\Else
    \State Use the scores of corresponding layers, retaining their nonnegative parts.
\EndIf
\State Compute the high-coverage count using Equation~\eeqref{eq:budgeted_allocation} and select the highest-scoring layers with the fixed tie rule.
\State Generate the layer-wise pair counts and frequency sets, and verify the total rotary budget.
\State Construct the fixed-carrier or physical MLA architecture and initialize target weights.
\State Fix the allocation and frequency sets, and pretrain with the target data and recipe.
\end{algorithmic}
\end{algorithm}

\subsection{Storage and Computational Complexity}
\label{sec:storage_compute_analysis}

We quantify BF16 weight-plus-KV-cache storage and matrix FLOPs, including configuration acquisition in total training cost.

\textbf{Fixed-carrier allocation.} Each layer retains its complete query/key projections and cache layout, while the rotary configuration changes the pairs receiving phase transformations. Coverage allocations therefore have the same parameter count, weight and cache storage, and principal matrix-multiplication dimensions. The additional computation for demand-guided configurations comes from early reference training and profiling.

\textbf{Physical MLA allocation.} With layer width, head count, content branch, and latent specifications fixed, rotary-query and positional-key projection parameters, positional-key storage, and the associated matrix multiplications scale linearly with total rotary budget \(B\). Uniform and demand-guided allocations at the same budget preserve these aggregate resource costs; reducing the budget decreases the corresponding parameters, storage, and computation. Layer-wise demand determines where these resources are placed, connecting the budget constraint to model quality.

\textbf{Early configuration acquisition and cross-scale reuse.} Let \(\mathcal C_{\mathrm{ref,full}}^{(s)}\), \(\mathcal C_{\mathrm{prof}}^{(s)}\), and \(\mathcal C_{\mathrm{train}}^{(t)}(\boldsymbol m)\) denote full reference training, source profiling, and target training costs, respectively. When one source profile serves \(A\) target configurations, the cost attributed to each target is
\begin{equation}
\mathcal C_{\mathrm{E2E}}(\boldsymbol m)
=
\mathcal C_{\mathrm{train}}^{(t)}(\boldsymbol m)
+
\frac{0.1\mathcal C_{\mathrm{ref,full}}^{(s)}+\mathcal C_{\mathrm{prof}}^{(s)}}{A}.
\label{eq:e2e_cost_main}
\end{equation}
Smaller models reduce reference-training and layer-wise probing computation, and profile reuse further amortizes configuration acquisition. Appendix~\ref{app:configuration_cost} derives the general cost relations; Appendix~\ref{app:experiment_details} provides storage and FLOPs counts for the experimental models and the resulting resource comparisons.

\section{Experiments}
\label{sec:experiments}

\subsection{Experimental Setup}
\label{sec:exp_setup}

We pretrain Dense Transformers~\citep{vaswani2017attention} and MLA Transformers~\citep{liu2024deepseekv2,liu2024deepseek} at 130M, 350M, and 1B scales on approximately 2.6B, 7B, and 20B tokens, respectively. Profiles are collected after \(10\%\) of reference training, and allocated target models are trained from scratch. The 130M experiments use same-scale profiles, while larger models reuse within-family 130M profiles; target-direct profiling is evaluated separately (Appendix~\ref{app:experiment_details}).

\textbf{Configurations and metrics.} Full/Uniform50 assign rotary widths 64/32 to every layer; Random50/Demand-guided assign width 64 to half the layers and zero to the rest, selected randomly or by early scores. Dense keeps query/key widths fixed, whereas MLA physically narrows rotary branches. Tables~\ref{tab:dense_main}--\ref{tab:mla_main} report validation loss and normalized storage/end-to-end FLOPs. Storage counts BF16 weights and a 4K KV cache; FLOPs include profiling costs.

\subsection{Dense Quality and Cost}
\label{sec:dense_results}

\begin{table}[t!]
\centering
\caption{Dense validation loss (top) and theoretical storage / end-to-end FLOPs ratios to Full (bottom).}
\label{tab:dense_main}
\setlength{\tabcolsep}{6pt}
\renewcommand{\arraystretch}{1.2}
\renewcommand{\ropeResult}[3]{%
  \shortstack[c]{#1\\[-2pt]{\tiny \(#2\,/\,#3\)}}%
}
\begin{threeparttable}
\begin{adjustbox}{max width=\columnwidth}
\begin{tabular}{@{}lcccc@{}}
\toprule
Model & Full & Uniform50 & Random50 &
\makecell[c]{Demand-\\guided} \\
\midrule
130M
& \ropeResult{3.374035}{1.000}{1.0000}
& \ropeResult{3.379236}{1.000}{1.0000}
& \ropeResult{3.378968}{1.000}{1.0000}
& \ropeResult{\textbf{3.368072}}{1.000}{1.1003}
\\[2pt]
350M
& \ropeResult{3.026883}{1.000}{1.0000}
& \ropeResult{3.032038}{1.000}{1.0000}
& \ropeResult{3.031077}{1.000}{1.0000}
& \ropeResult{\textbf{3.023735}}{1.000}{1.0147}
\\[2pt]
1B
& \ropeResult{2.785044}{1.000}{1.0000}
& \ropeResult{2.797076}{1.000}{1.0000}
& \ropeResult{2.792187}{1.000}{1.0000}
& \ropeResult{\textbf{2.757636}}{1.000}{1.0019}
\\
\bottomrule
\end{tabular}
\end{adjustbox}
\end{threeparttable}
\end{table}

\begin{table}[t!]
\centering
\caption{MLA validation loss (top) and theoretical storage / end-to-end FLOPs ratios to Full (bottom).}
\label{tab:mla_main}
\setlength{\tabcolsep}{6pt}
\renewcommand{\arraystretch}{1.2}
\renewcommand{\ropeResult}[3]{%
  \shortstack[c]{#1\\[-2pt]{\tiny \(#2\,/\,#3\)}}%
}
\begin{threeparttable}
\begin{adjustbox}{max width=\columnwidth}
\begin{tabular}{@{}lcccc@{}}
\toprule
Model & Full & Uniform50 &
\makecell[c]{Random50\\mean} &
\makecell[c]{Demand-\\guided} \\
\midrule
130M
& \ropeResult{\textbf{3.377954}}{1.000}{1.0000}
& \ropeResult{3.425247}{0.961}{0.9302}
& \ropeResult{3.412386}{0.961}{0.9302}
& \ropeResult{3.381707}{0.961}{1.0305}
\\[2pt]
350M
& \ropeResult{3.052147}{1.000}{1.0000}
& \ropeResult{3.067263}{0.965}{0.9326}
& \ropeResult{3.062779}{0.965}{0.9326}
& \ropeResult{\textbf{3.051440}}{0.965}{0.9476}
\\[2pt]
1B
& \ropeResult{\textbf{2.750373}}{1.000}{1.0000}
& \shortstack[c]{NA\\[-2pt]{\tiny\phantom{0.000/0.0000}}}
& \shortstack[c]{NA\\[-2pt]{\tiny\phantom{0.000/0.0000}}}
& \ropeResult{2.752074}{0.964}{0.9396}
\\
\bottomrule
\end{tabular}
\end{adjustbox}
\end{threeparttable}
\end{table}

\textbf{Allocation quality.} Demand-guided allocation outperforms Full, Uniform50, and Random50 at all three scales (Table~\ref{tab:dense_main}). Since query/key widths and storage remain unchanged, the gains reflect how rotary coverage is placed across depth, rather than changes in model capacity. The strongest improvement occurs at 1B using a 130M source profile. Neither uniform half coverage nor random placement consistently outperforms Full at these scales.

\textbf{Configuration cost.} Fixed-carrier allocation preserves matrix-multiplication dimensions. With source acquisition included, end-to-end FLOPs are \(1.0147\times\) and \(1.0019\times\) Full at 350M and 1B; model storage is unchanged.

\subsection{MLA Quality and Physical Resources}
\label{sec:mla_results}

\textbf{Matched-budget comparisons.} MLA physically narrows rotary-query and positional-key branches, unlike fixed-carrier Dense. At both 130M and 350M, demand-guided allocation beats uniform and random half-budget controls. Their equal aggregate rotary width isolates the effect of layer placement (Table~\ref{tab:mla_main}).

\textbf{Quality--efficiency trade-off.} At 350M, the transferred configuration slightly improves on Full while using \(0.965\times\) storage and \(0.9476\times\) end-to-end FLOPs. At 1B, it nearly matches Full with \(0.964\times\) storage and \(0.9396\times\) FLOPs; budget-matched controls are unavailable at this scale. The 130M configuration also approaches Full, although same-scale profile acquisition raises its end-to-end FLOPs ratio to \(1.0305\times\). Thus, net computational savings depend on the cost of acquiring the source profile for architecture selection.

\subsection{Layer Placement and Rotary Budget}
\label{sec:exp_layer_budget}

\textbf{Layer-placement ablations.} At 130M MLA, demand-guided allocation outperforms reversed ranking, last-ten-layer placement, and the mean of three random layer sets; transferred 350M MLA also beats the mean of two random layer sets. These matched-budget tests favor demand-guided placement over arbitrary alternatives (Appendix~\ref{app:layer_ablation}).

\textbf{Budget ablation.} Profile50 improves on Profile75 at 130M MLA despite using fewer rotary pairs, indicating a non-monotone validation-loss response. The ranking is fixed, while the number of high-coverage layers changes across budgets (Appendix~\ref{app:budget_ablation}).

\subsection{Cross-Scale Transfer and Replication}
\label{sec:exp_transfer_repeat}

\textbf{Source-profile transfer.} A 130M early profile determines 350M and 1B allocations without target-scale probing. Both 350M Dense and MLA improve over matched-budget controls, with acquisition costs included in Tables~\ref{tab:dense_main}--\ref{tab:mla_main}.

\textbf{Paired replications.} Across two paired 350M MLA training groups, Transfer50 improves over Uniform50 by \(0.011348\) validation loss on average. A separate target-direct comparison also improves 1B Dense against its paired baseline. Comparisons are made within training groups before averaging (Appendix~\ref{app:paired_transfer}).

\section{Conclusion}
\label{sec:conclusion}

We presented profiling-guided layer-wise RoPE allocation, translating model-specific positional demand into budgeted architecture design. Early profiles support within-family cross-scale transfer, with experiments showing matched-budget validation loss improvements and resource analysis identifying favorable quality--storage--compute trade-offs for physical MLA.

\section{Limitations}
\label{sec:limitations}

Our fixed high/low design simplifies architecture construction, while its coarse granularity leaves room for finer resource allocation. Transfer relies on within-family structural correspondence, motivating further exploration of reusable early profiles across broader architectures and longer contexts.

{
\small
\bibliography{reference}
\bibliographystyle{abbrv}
}

\newpage
\appendix
\begin{center}
\Large
\textbf{Appendix}
\end{center}
\vspace{-2mm}
\tableofcontents
\vspace{3mm}
\newpage
% Required packages: amsmath, amssymb, amsthm.
% Reuse the existing theorem, lemma, and proof environments.

\section{Analysis of Profiling-Guided Rotary Allocation}
\label{app:allocation_analysis}

\subsection{Setup and Reference-Consistent Loss Decomposition}
\label{app:allocation_setup}

\textbf{Feasible high/low configurations.} Let \([L]=\{1,\ldots,L\}\). At layer \(\ell\), choose integer pair counts \(0\leq m_{\ell,L}<m_{\ell,H}=M_\ell\) with a common increment \(m_{\ell,H}-m_{\ell,L}=\Delta m>0\). Fix an integer \(1\leq K<L\) and a budget
\begin{equation}
B=\sum_{\ell=1}^{L}m_{\ell,L}+K\Delta m.
\label{eq:app_highlow_budget}
\end{equation}
For a high-coverage set \(S\subseteq[L]\), define
\begin{equation}
\begin{gathered}
m_\ell(S)=\begin{cases}
m_{\ell,H},&\ell\in S,\\
m_{\ell,L},&\ell\notin S,
\end{cases}\quad\mathcal S_K=\{S\subseteq[L]:|S|=K\}.
\end{gathered}
\label{eq:app_highlow_configuration}
\end{equation}
Indeed, \(\sum_\ell m_\ell(S)=\sum_\ell m_{\ell,L}+|S|\Delta m\), so the budget in Equation~\eeqref{eq:app_highlow_budget} is equivalent to \(S\in\mathcal S_K\). The low coverage at layer \(\ell\) is \(\rho_{\ell,L}=m_{\ell,L}/M_\ell\), and the high coverage is one. Each layer uses the nested frequency-pair subsets specified in Section~\ref{sec:preliminaries}.

\textbf{Frozen-checkpoint objective.} Fix model weights and an evaluation distribution \(\mathcal P\), which may be the empirical distribution on a fixed evaluation sample set. Let \(\mathcal L(x;\boldsymbol m(S))\) denote the mean per-token validation loss on text sample \(x\) under allocation \(\boldsymbol m(S)\), and define
\begin{equation}
F(S)=\mathbb E_{x\sim\mathcal P}\bigl[\mathcal L(x;\boldsymbol m(S))\bigr],
\qquad S\subseteq[L].
\label{eq:app_configuration_loss}
\end{equation}
Assume these losses are finite. Within each model, all evaluations use its fixed intervention rule: unselected frequency pairs receive identity rotation in Dense, while unselected positional coordinates are zeroed in MLA, as specified in Appendix~\ref{app:exp_setup}. The shared frequency bank and attention scaling remain fixed across allocations. Write \(F_H=F([L])\) for the all-high reference loss and define
\begin{equation}
G_\ell=F([L]\setminus\{\ell\})-F_H.
\label{eq:app_reference_upgrade}
\end{equation}
Thus, \(G_\ell\) is the benefit of restoring layer \(\ell\) from low to high coverage while all other layers remain high. For the identity intervention with the same checkpoint and \(\mathcal P\) equal to the empirical distribution on \(\mathcal D_{\mathrm{cal}}\), this recovers \(\widehat G_\ell(\rho_{\ell,L},1)\) in Equation~\eeqref{eq:coverage_upgrade_gain}.

\textbf{Joint-intervention residual.} For every \(S\subseteq[L]\), set
\begin{align}
I(S)&=F(S)-F_H-\sum_{\ell\notin S}G_\ell,
\label{eq:app_interaction_definition}\\
C&=F_H+\sum_{\ell=1}^{L}G_\ell.
\label{eq:app_allocation_constant}
\end{align}
These definitions give the exact identity
\begin{equation}
F(S)=C-\sum_{\ell\in S}G_\ell+I(S).
\label{eq:app_anchored_decomposition}
\end{equation}
The residual \(I(S)\) measures the difference between the joint loss change and the sum of single-layer loss changes at the same all-high reference. In particular, \(I([L])=0\) and \(I([L]\setminus\{\ell\})=0\). For a joint change, its magnitude can be controlled through mixed finite differences, as shown next.

\begin{lemma}[Control of joint-intervention effects]
\label{lem:joint_intervention_control}
Let \(q=L-K\), define \(f(T)=F([L]\setminus T)\), and suppose there are symmetric constants \(b_{ij}=b_{ji}\geq0\) such that
\begin{equation}
\left|\begin{aligned}
&f(U\cup\{i,j\})-f(U\cup\{i\})\\
&\quad-f(U\cup\{j\})+f(U)
\end{aligned}\right|\leq b_{ij}
\label{eq:app_mixed_difference_bound}
\end{equation}
for all distinct \(i,j\) and every \(U\subseteq[L]\setminus\{i,j\}\). Then every \(S\in\mathcal S_K\) satisfies
\begin{equation}
|I(S)|\leq\sum_{\substack{i<j\\ i,j\notin S}}b_{ij}.
\label{eq:app_interaction_pair_bound}
\end{equation}
Consequently, one may take
\begin{equation}
\delta=\max_{S\in\mathcal S_K}\sum_{\substack{i<j\\ i,j\notin S}}b_{ij}.
\label{eq:app_interaction_delta}
\end{equation}
If \(b_{ij}\leq b\) for all pairs, then \(\delta\leq\binom{q}{2}b\).
\end{lemma}

\begin{proof}
\textbf{Step 1: Express the joint change as sequential interventions.} Fix \(S\in\mathcal S_K\), write its complement as \(T=[L]\setminus S=\{i_1,\ldots,i_q\}\), and set \(T_p=\{i_1,\ldots,i_p\}\), with \(T_0=\varnothing\). For \(i\notin U\), define the first finite difference \(\Delta_i f(U)=f(U\cup\{i\})-f(U)\). Since \(G_i=\Delta_i f(\varnothing)\), telescoping gives
\begin{align}
f(T)-f(\varnothing)&=\sum_{p=1}^{q}\Delta_{i_p}f(T_{p-1}),\\
I(S)&=\sum_{p=1}^{q}\left[\Delta_{i_p}f(T_{p-1})-\Delta_{i_p}f(\varnothing)\right].
\label{eq:app_interaction_telescope}
\end{align}

\textbf{Step 2: Expand each background-dependent change.} For \(p\geq2\), telescope along the previously changed layers:
\begin{align}
&\Delta_{i_p}f(T_{p-1})-\Delta_{i_p}f(\varnothing)=\sum_{a=1}^{p-1}\left[\Delta_{i_p}f(T_a)-\Delta_{i_p}f(T_{a-1})\right]=\sum_{a=1}^{p-1}\left[\begin{aligned}
&f(T_{a-1}\cup\{i_a,i_p\})\\
&-f(T_{a-1}\cup\{i_a\})\\
&-f(T_{a-1}\cup\{i_p\})+f(T_{a-1})
\end{aligned}\right].
\label{eq:app_mixed_telescope}
\end{align}
Each bracket is a mixed finite difference on a valid background, since \(i_a,i_p\notin T_{a-1}\).

\textbf{Step 3: Sum the pairwise bounds.} Applying Equation~\eeqref{eq:app_mixed_difference_bound} to every bracket and then using the triangle inequality yields
\begin{equation}
|I(S)|\leq\sum_{p=2}^{q}\sum_{a=1}^{p-1}b_{i_a i_p}
=\sum_{\substack{i<j\\ i,j\notin S}}b_{ij}.
\end{equation}
Every unordered pair in \(T\) appears exactly once. Taking the maximum over \(S\in\mathcal S_K\) proves Equation~\eeqref{eq:app_interaction_delta}; bounding each of its \(\binom{q}{2}\) summands by \(b\) gives the final assertion. For \(q=1\), the sums are empty and \(I(S)=0\).
\end{proof}

\subsection{Calibrated Profiles and Budgeted Layer Selection}
\label{app:calibrated_allocation}

The following theorem formalizes Theorem~\ref{thm:calibrated_allocation_main}. For the allocation method, take \(s_\ell=[a_\ell]_+\), or a common positive normalization of these priorities, where \(a_\ell\) is the recorded matched-marginal score obtained from Equation~\eeqref{eq:matched_marginal_score} under the intervention rule fixed above. The calibration condition relates these priorities to the all-high reference gains in Equation~\eeqref{eq:app_reference_upgrade}; its error \(\varepsilon\) accounts for differences in probing background and coverage interval, together with score estimation and processing.

\begin{theorem}[Allocation loss under calibrated scores]
\label{thm:calibrated_allocation_full}
For the loss and gains defined above, suppose a score vector \(\boldsymbol s\in\mathbb R^L\) and constants \(\alpha>0\), \(\varepsilon\geq0\), and \(\delta\geq0\) satisfy
\begin{equation}
\begin{aligned}
|G_\ell-\alpha s_\ell|&\leq\varepsilon&&\text{for all }\ell\in[L],\\
|I(S)|&\leq\delta&&\text{for all }S\in\mathcal S_K.
\end{aligned}
\label{eq:app_calibration_conditions}
\end{equation}
Let \(\widehat S\) contain the \(K\) largest scores, with ties resolved by a fixed rule. Let \(S^\star\) minimize \(F\) over \(\mathcal S_K\). If \(r=|S^\star\setminus\widehat S|\), then
\begin{equation}
0\leq F(\widehat S)-F(S^\star)
\leq2r\varepsilon+2\delta
\leq2K\varepsilon+2\delta.
\label{eq:app_allocation_regret}
\end{equation}
In particular, \(\varepsilon=\delta=0\) implies \(F(\widehat S)=\min_{S\in\mathcal S_K}F(S)\).
\end{theorem}

\begin{proof}
\textbf{Step 1: Establish the top-\(K\) score inequality.} Since \(\mathcal S_K\) is finite and \(F\) is finite, a minimizing set \(S^\star\) exists. Define the disjoint sets \(S_+=S^\star\setminus\widehat S\) and \(S_-=\widehat S\setminus S^\star\). Both have cardinality \(r\), because \(|S^\star|=|\widehat S|=K\). By the top-\(K\) construction, every score indexed by \(S_-\) is at least every score indexed by \(S_+\). Pairing their elements therefore gives
\begin{equation}
\sum_{\ell\in S^\star}s_\ell-\sum_{\ell\in\widehat S}s_\ell
=\sum_{\ell\in S_+}s_\ell-\sum_{\ell\in S_-}s_\ell\leq0.
\label{eq:app_topk_inequality}
\end{equation}
This also holds when \(r=0\), in which case both differences vanish.

\textbf{Step 2: Translate score ordering into a gain bound.} Write \(e_\ell=G_\ell-\alpha s_\ell\), so that \(|e_\ell|\leq\varepsilon\). The gains on the shared indices cancel. Using \(\alpha>0\) and Equation~\eeqref{eq:app_topk_inequality}, we obtain
\begin{align}
&\sum_{\ell\in S^\star}G_\ell-\sum_{\ell\in\widehat S}G_\ell=\alpha\left(\sum_{\ell\in S^\star}s_\ell-\sum_{\ell\in\widehat S}s_\ell\right)+\sum_{\ell\in S_+}e_\ell-\sum_{\ell\in S_-}e_\ell=\sum_{\ell\in S_+}|e_\ell|+\sum_{\ell\in S_-}|e_\ell|\leq2r\varepsilon.
\label{eq:app_gain_comparison}
\end{align}

\textbf{Step 3: Account for joint-intervention effects.} Applying the exact decomposition in Equation~\eeqref{eq:app_anchored_decomposition} to both configurations gives
\begin{align}
&F(\widehat S)-F(S^\star)=\sum_{\ell\in S^\star}G_\ell-\sum_{\ell\in\widehat S}G_\ell+I(\widehat S)-I(S^\star)\leq2r\varepsilon+|I(\widehat S)|+|I(S^\star)|\leq2r\varepsilon+2\delta.
\end{align}
Since \(r\leq K\), this proves the upper bound. The lower bound follows from the optimality of \(S^\star\).

\textbf{Step 4: Recover the exact-calibration case.} If \(\varepsilon=\delta=0\), the upper and lower bounds both equal zero. Thus \(\widehat S\) minimizes the configuration loss over \(\mathcal S_K\), proving the final assertion.
\end{proof}

The following consequence makes the connection between layer-wise demand differences and budget-preserving configuration changes explicit.

\begin{lemma}[Benefit of a demand-guided layer exchange]
\label{lem:demand_guided_exchange}
Under the conditions of Theorem~\ref{thm:calibrated_allocation_full}, let \(S\in\mathcal S_K\), choose \(i\notin S\) and \(j\in S\), and set \(S'=(S\setminus\{j\})\cup\{i\}\). Then \(S'\in\mathcal S_K\) and
\begin{equation}
F(S)-F(S')\geq\alpha(s_i-s_j)-2\varepsilon-2\delta.
\label{eq:app_exchange_gain}
\end{equation}
Consequently, \(\alpha(s_i-s_j)>2\varepsilon+2\delta\) implies \(F(S')<F(S)\).
\end{lemma}

\begin{proof}
\textbf{Step 1: Verify feasibility of the exchange.} Since \(j\in S\) and \(i\notin S\), removing \(j\) and adding \(i\) preserves cardinality, giving \(|S'|=|S|=K\). The common upgrade increment \(\Delta m\) therefore gives
\begin{equation}
\sum_{\ell=1}^{L}m_\ell(S')=\sum_{\ell=1}^{L}m_{\ell,L}+K\Delta m=B.
\end{equation}
Thus, \(S'\) is feasible under the same rotary budget.

\textbf{Step 2: Express the exact loss change.} Apply Equation~\eeqref{eq:app_anchored_decomposition} to \(S\) and \(S'\). Their shared indices cancel, yielding
\begin{align}
&F(S)-F(S')=\sum_{\ell\in S'}G_\ell-\sum_{\ell\in S}G_\ell+I(S)-I(S')=G_i-G_j+I(S)-I(S').
\label{eq:app_exchange_identity}
\end{align}

\textbf{Step 3: Bound calibration and interaction errors.} The calibration conditions imply \(G_i\geq\alpha s_i-\varepsilon\) and \(G_j\leq\alpha s_j+\varepsilon\). Since both sets are feasible, \(I(S)\geq-\delta\) and \(I(S')\leq\delta\). Substituting these inequalities into Equation~\eeqref{eq:app_exchange_identity} gives
\begin{align}
F(S)-F(S')\geq(\alpha s_i-\varepsilon)-(\alpha s_j+\varepsilon)-2\delta=\alpha(s_i-s_j)-2\varepsilon-2\delta.
\end{align}
This proves Equation~\eeqref{eq:app_exchange_gain}. If \(\alpha(s_i-s_j)>2\varepsilon+2\delta\), the right-hand side is strictly positive, and hence \(F(S')<F(S)\).
\end{proof}

\subsection{Normalized-Depth Profile Transfer}
\label{app:profile_transfer}

\textbf{Profiles and interpolation.} Consider a source model with \(L_s\) layers and a target model with \(L_t\) layers in the same family. For \(u\in\{s,t\}\), define normalized depth and a normalized reference allocation profile by
\begin{equation}
\begin{aligned}
x_\ell^{(u)}&=\frac{\ell-\tfrac12}{L_u},
\qquad s_\ell^{(u)}=\frac{[\bar a_\ell^{(u)}]_+}{\mu_u},\qquad \mu_u=\frac{1}{L_u}\sum_{j=1}^{L_u}[\bar a_j^{(u)}]_+>0.
\end{aligned}
\label{eq:app_normalized_profiles}
\end{equation}
where \(\bar a_\ell^{(u)}\) is the population paired allocation score at the fixed reference checkpoint, obtained by replacing the sample average in Equation~\eeqref{eq:matched_marginal_score} with expectation under the chosen evaluation distribution and model-specific intervention rule. Let \(\widehat{\boldsymbol s}^{(s)}\) be the recorded source allocation profile, normalized as in Equation~\eeqref{eq:source_profile_transfer}, with a positive denominator. The source-profile error below includes sampling and score aggregation. Interpolation from source depths to target depths is represented by \(W\in\mathbb R^{L_t\times L_s}\) satisfying
\begin{equation}
W_{ji}\geq0,
\qquad
\sum_{i=1}^{L_s}W_{ji}=1,
\qquad
\widetilde{\boldsymbol s}=W\widehat{\boldsymbol s}^{(s)}.
\label{eq:app_profile_interpolation}
\end{equation}
Piecewise-linear interpolation with constant extension beyond the source endpoints has these properties: an interior target depth uses nonnegative weights on its two neighboring source depths, and an exterior depth uses the nearest source endpoint. The target allocation selects the \(K_t\) largest entries of \(\widetilde{\boldsymbol s}\).

\begin{theorem}[Allocation loss under profile transfer]
\label{thm:profile_transfer_full}
Let \(F_t\), \(G_j^{(t)}\), and \(I_t\) be the target-model quantities defined in Appendix~\ref{app:allocation_setup}, with feasible sets \(\mathcal S_{K_t}\). Suppose
\begin{equation}
|G_j^{(t)}-\alpha_t s_j^{(t)}|\leq\varepsilon_t,
\qquad
|I_t(S)|\leq\delta_t
\label{eq:app_target_calibration}
\end{equation}
for every target layer \(j\) and every \(S\in\mathcal S_{K_t}\), where \(\alpha_t>0\). Suppose also that the normalized source measurement error and depth-mapped profile discrepancy satisfy
\begin{equation}
\begin{aligned}
\|\widehat{\boldsymbol s}^{(s)}-\boldsymbol s^{(s)}\|_\infty&\leq\varepsilon_s,\quad
\|W\boldsymbol s^{(s)}-\boldsymbol s^{(t)}\|_\infty\leq\eta_{\mathrm{depth}}.
\end{aligned}
\label{eq:app_transfer_errors}
\end{equation}
Let \(\widehat S_{s\to t}\) be a top-\(K_t\) set of \(\widetilde{\boldsymbol s}\), let \(S_t^\star\) minimize \(F_t\) over \(\mathcal S_{K_t}\), and define \(\eta=\|\widetilde{\boldsymbol s}-\boldsymbol s^{(t)}\|_\infty\). Then \(\eta\leq\varepsilon_s+\eta_{\mathrm{depth}}\), and
\begin{align}
&0\leq F_t(\widehat S_{s\to t})-F_t(S_t^\star)\leq2K_t(\varepsilon_t+\alpha_t\eta)+2\delta_t\leq2K_t\bigl[\varepsilon_t+\alpha_t(\varepsilon_s+\eta_{\mathrm{depth}})\bigr]+2\delta_t.
\label{eq:app_transfer_regret}
\end{align}
\end{theorem}

\begin{proof}
\textbf{Step 1: Control propagation of source measurement error.} For any vector \(\boldsymbol v\in\mathbb R^{L_s}\), nonnegativity and unit row sums of \(W\) imply, for each target layer \(j\),
\begin{equation}
\begin{aligned}
|(W\boldsymbol v)_j|&\leq\sum_{i=1}^{L_s}W_{ji}|v_i|\leq\|\boldsymbol v\|_\infty\sum_{i=1}^{L_s}W_{ji}
=\|\boldsymbol v\|_\infty.
\end{aligned}
\end{equation}
Taking the maximum over \(j\) yields \(\|W\boldsymbol v\|_\infty\leq\|\boldsymbol v\|_\infty\). With \(\boldsymbol v=\widehat{\boldsymbol s}^{(s)}-\boldsymbol s^{(s)}\), this gives
\begin{equation}
\|W\widehat{\boldsymbol s}^{(s)}-W\boldsymbol s^{(s)}\|_\infty\leq\varepsilon_s.
\label{eq:app_interpolation_error}
\end{equation}

\textbf{Step 2: Bound the complete transferred-profile discrepancy.} Add and subtract \(W\boldsymbol s^{(s)}\), and apply the triangle inequality:
\begin{align}
\eta
&=\|W\widehat{\boldsymbol s}^{(s)}-\boldsymbol s^{(t)}\|_\infty\leq\|W(\widehat{\boldsymbol s}^{(s)}-\boldsymbol s^{(s)})\|_\infty+\|W\boldsymbol s^{(s)}-\boldsymbol s^{(t)}\|_\infty\leq\varepsilon_s+\eta_{\mathrm{depth}}.
\label{eq:app_total_profile_error}
\end{align}

\textbf{Step 3: Calibrate the transferred scores to target benefits.} For every target layer \(j\),
\begin{align}
|G_j^{(t)}-\alpha_t\widetilde s_j|
&\leq|G_j^{(t)}-\alpha_t s_j^{(t)}|+\alpha_t|s_j^{(t)}-\widetilde s_j|\leq\varepsilon_t+\alpha_t\eta.
\label{eq:app_transferred_calibration}
\end{align}
Thus the transferred profile satisfies the calibration condition of Theorem~\ref{thm:calibrated_allocation_full} with error \(\varepsilon_t+\alpha_t\eta\).

\textbf{Step 4: Apply the allocation bound at the target scale.} The target budget defines \(\mathcal S_{K_t}\), the selected set maximizes the sum of transferred scores, and the interaction bound in Equation~\eeqref{eq:app_target_calibration} holds on all feasible sets. Theorem~\ref{thm:calibrated_allocation_full} therefore gives
\begin{equation}
0\leq F_t(\widehat S_{s\to t})-F_t(S_t^\star)
\leq2K_t(\varepsilon_t+\alpha_t\eta)+2\delta_t.
\end{equation}
Substituting Equation~\eeqref{eq:app_total_profile_error} proves the second inequality in Equation~\eeqref{eq:app_transfer_regret}. This also proves the transfer statement in the main text for any transferred profile with discrepancy at most \(\eta\).
\end{proof}

\section{Cost Analysis of Early Profiling and Cross-Scale Configuration}
\label{app:configuration_cost}

\subsection{Early Reference Training and Layer-Wise Probing}
\label{app:early_profile_cost}

We derive the computational components of configuration acquisition. For model \(a\), let \(L_a\) be its depth, \(N_a\) its full training-token budget, \(f_a\) its forward matrix-multiplication FLOPs per token, and \(\tau\) the reference-training fraction.

\textbf{Early reference training.} Counting the forward and corresponding backward matrix multiplications as three forward passes gives the full reference-training cost \(\mathcal C_{\mathrm{ref,full}}^{(a)}=3N_af_a\). With a fixed training batch, a fraction \(\tau\) of training steps processes the same fraction of training tokens, yielding
\begin{equation}
\mathcal C_{\mathrm{early}}^{(a)}=3\tau N_af_a.
\label{eq:early_reference_cost}
\end{equation}
We use \(\tau=0.1\), giving \(\mathcal C_{\mathrm{early}}^{(a)}=0.3N_af_a\).

\textbf{Layer-wise interventions.} Phase-off measurement performs one reference forward pass and \(L_a\) single-layer interventions for each evaluation input batch. Low/high-coverage benefit measurement performs one reference pass and two interventions per layer. Their respective call counts are
\begin{equation}
q_{\mathrm{phase}}=L_a+1,
\qquad
q_{\mathrm{gain}}=2L_a+1.
\label{eq:profile_query_counts}
\end{equation}
Suppose the selected profile is constructed from \(E\) probing events, with global sequence count \(b_e\) and sequence length \(T_e\) at event \(e\). The cumulative input-token count is
\begin{equation}
N_{\mathrm{prof}}^{(a)}=\sum_{e=1}^{E}b_eT_e.
\label{eq:profile_token_budget}
\end{equation}
On a fixed-width intervention carrier, the corresponding probing costs are
\begin{equation}
\begin{aligned}
\mathcal C_{\mathrm{phase}}^{(a)}&=(L_a+1)N_{\mathrm{prof}}^{(a)}f_a,\qquad \mathcal C_{\mathrm{gain}}^{(a)}=(2L_a+1)N_{\mathrm{prof}}^{(a)}f_a.
\end{aligned}
\label{eq:profiling_costs}
\end{equation}
The cumulative input count includes all probing events used to construct the final profile. Repeated evaluation of the same inputs across events is counted according to the executed forward passes.

\textbf{Configuration acquisition.} Combining early reference training with the chosen probing procedure gives
\begin{equation}
\mathcal C_{\mathrm{config}}^{(a)}=3\tau N_af_a+\mathcal C_{\mathrm{prof}}^{(a)}.
\label{eq:configuration_cost}
\end{equation}

\subsection{Same-Scale and Cross-Scale Configuration Acquisition}
\label{app:cross_scale_cost}

Consider a source model \(s\) and a target model \(t\) from the same family, both profiled at the same reference-training fraction \(\tau\).

\textbf{Reference-training cost ratio.} Equation~\eeqref{eq:early_reference_cost} gives
\begin{equation}
r_{\mathrm{early}}^{s\to t}
=
\frac{\mathcal C_{\mathrm{early}}^{(s)}}{\mathcal C_{\mathrm{early}}^{(t)}}
=
\frac{N_sf_s}{N_tf_t}.
\label{eq:early_cost_ratio}
\end{equation}
The reference-training cost therefore depends jointly on the training-token budget and per-token model arithmetic.

\textbf{Probing cost ratio.} For low/high-coverage benefit measurement,
\begin{equation}
r_{\mathrm{prof}}^{s\to t}
=
\frac{(2L_s+1)N_{\mathrm{prof}}^{(s)}f_s}{(2L_t+1)N_{\mathrm{prof}}^{(t)}f_t}.
\label{eq:profile_cost_ratio_general}
\end{equation}
With the same cumulative probing input, this simplifies to
\begin{equation}
r_{\mathrm{prof}}^{s\to t}
=
\frac{(2L_s+1)f_s}{(2L_t+1)f_t}.
\label{eq:profile_cost_ratio_same}
\end{equation}
This relation captures the contributions of model depth and single-pass computation to profiling cost.

\textbf{Combined cost ratio.} Combining reference training and probing yields
\begin{equation}
\frac{\mathcal C_{\mathrm{config}}^{(s)}}{\mathcal C_{\mathrm{config}}^{(t)}}
=
\frac{r_{\mathrm{early}}^{s\to t}\mathcal C_{\mathrm{early}}^{(t)}+r_{\mathrm{prof}}^{s\to t}\mathcal C_{\mathrm{prof}}^{(t)}}{\mathcal C_{\mathrm{early}}^{(t)}+\mathcal C_{\mathrm{prof}}^{(t)}}.
\label{eq:configuration_cost_ratio}
\end{equation}
The two cost weights are nonnegative and sum to one after normalization. The total configuration-cost ratio therefore lies between the component ratios and satisfies
\begin{equation}
\frac{\mathcal C_{\mathrm{config}}^{(s)}}{\mathcal C_{\mathrm{config}}^{(t)}}
\leq
\max\left\{r_{\mathrm{early}}^{s\to t},r_{\mathrm{prof}}^{s\to t}\right\}.
\label{eq:configuration_cost_upper}
\end{equation}
This derivation expresses cross-scale configuration efficiency through two measurable components: early reference training and layer-wise probing.

\subsection{End-to-End Cost and Profile Reuse}
\label{app:cost_amortization}

For target allocation \(\boldsymbol m\), let \(f_t(\boldsymbol m)\) denote forward FLOPs per token. Target pretraining from scratch costs
\begin{equation}
\mathcal C_{\mathrm{train}}^{(t)}(\boldsymbol m)=3N_tf_t(\boldsymbol m).
\label{eq:target_training_cost}
\end{equation}
When one source profile supports \(A\) target configurations and its acquisition cost is divided equally, the end-to-end cost attributed to each target is
\begin{equation}
\mathcal C_{\mathrm{E2E}}(\boldsymbol m)
=
3N_tf_t(\boldsymbol m)+\frac{\mathcal C_{\mathrm{config}}^{(s)}}{A}.
\label{eq:e2e_cost}
\end{equation}
Normalizing by full-rotation target training at the same specification, \(\mathcal C_{\mathrm{base}}^{(t)}=3N_tf_t^{\mathrm{full}}\), gives
\begin{equation}
r_{\mathrm{E2E}}
=
\frac{f_t(\boldsymbol m)}{f_t^{\mathrm{full}}}
+
\frac{3\tau N_sf_s+\mathcal C_{\mathrm{prof}}^{(s)}}{A\cdot3N_tf_t^{\mathrm{full}}}.
\label{eq:e2e_ratio_general}
\end{equation}
The first term is determined by the target architecture, and the second by the profile source and reuse count. Fixed-carrier allocation preserves the principal matrix-multiplication count in the first term, while physical dimension allocation changes it through the target structure. Cross-scale transfer and profile reuse jointly reduce the second term.

If the same source profile supports \(A\) targets with different architectures, the total computation is
\begin{equation}
\sum_{j=1}^{A}\mathcal C_{\mathrm{train}}^{(t_j)}(\boldsymbol m^{(t_j)})
+
\mathcal C_{\mathrm{config}}^{(s)}.
\label{eq:multi_target_total}
\end{equation}
Under equal attribution across targets, the assigned source costs sum to the cost of acquiring the profile once.

\section{Experimental Settings, Resource Accounting, and Additional Results}
\label{app:experiment_details}

\subsection{Models, Training, and Profiling Settings}
\label{app:exp_setup}

\textbf{Models for positional-demand analysis.} Section~\ref{sec:positional_demand} analyzes DeepSeek-MoE-16B~\citep{dai2024deepseekmoe}, LLaMA-2-7B~\citep{touvron2023llama}, LLaMA-3.2-3B~\citep{meta2024llama32}, and Qwen3-30B~\citep{yang2025qwen3}. Phase-off profiling uses 500 shared text samples per model, and coverage calibration uses 100 samples per configuration. Within each measurement, reference and intervention configurations share input token sequences, attention masks, and loss-evaluation positions. Analysis figures use one-based layer numbers; frozen architecture sets in Table~\ref{tab:active_layers} use zero-based indices.

\textbf{Models pretrained from scratch.} Table~\ref{tab:model_specs} summarizes the model structures. Shared settings include sequence length 4,096, padded vocabulary size 128,896, tied input/output embeddings, RMSNorm, SwiGLU, BF16, and zero attention and hidden dropout. The RoPE base is 10,000, and RoPE fusion is disabled. In Table~\ref{tab:model_specs}, \(L,d,d_{\mathrm{ff}},h,d_c\) denote depth, hidden width, FFN width, head count, and MLA KV latent width. Dense uses 64-dimensional queries, keys, and values per head. MLA uses content, full rotary, and value widths of 64 per head, with Q-LoRA disabled.

\begin{table}[t!]
\centering
\setlength{\tabcolsep}{6pt}
\renewcommand{\arraystretch}{1.2}
\caption{Pretraining architecture dimensions by scale; \(d_c\) is the MLA KV latent width.}
\label{tab:model_specs}
\begin{threeparttable}
\begin{adjustbox}{max width=\columnwidth}
\begin{tabular}{lccccc}
\toprule
Scale & $L$ & $d$ & $d_{\mathrm{ff}}$ & $h$ & $d_c$ \\
\midrule
130M & 20 & 512 & 1,408 & 8 & 128 \\
350M & 24 & 896 & 2,432 & 14 & 256 \\
1B, main & 28 & 1,536 & 4,096 & 24 & 512 \\
1B, KV384 & 28 & 1,536 & 4,096 & 24 & 384 \\
\bottomrule
\end{tabular}
\end{adjustbox}
\end{threeparttable}
\end{table}

\textbf{Optimization and execution.} Distributed target training is based on Megatron-LM~\citep{shoeybi2019megatron}. We use Adam with \(\beta_1=0.9\), \(\beta_2=0.95\), weight decay 0.1, and gradient-clipping threshold 1. The learning rate follows a warmup--stable--decay schedule with cosine decay in the final phase. Table~\ref{tab:training_hparams} reports the scale-specific settings; token counts equal steps times global batch size times sequence length.

\begin{table*}[t!]
\centering
\small
\setlength{\tabcolsep}{6pt}
\renewcommand{\arraystretch}{1.14}
\caption{Pretraining hyperparameters by scale, including token budgets, batch sizes, and learning-rate schedules.}
\label{tab:training_hparams}
\begin{threeparttable}
\begin{tabular*}{\textwidth}{@{\extracolsep{\fill}}lccc@{}}
\toprule
Setting & 130M & 350M & 1B \\
\midrule
Sequence length & 4,096 & 4,096 & 4,096 \\
GPUs per job & 2 & 4 & 4 or 8 \\
Micro-batch size per GPU & 4 & 2 & 1 \\
Global batch size & 64 & 128 & 256 \\
Tokens per step & 262,144 & 524,288 & 1,048,576 \\
Training steps & 9,919 & 13,352 & 19,074 \\
Training tokens & 2,600,206,336 & 7,000,293,376 & 20,000,538,624 \\
Peak learning rate & \(6\times10^{-4}\) & \(4\times10^{-4}\) & \(3\times10^{-4}\) \\
Minimum learning rate & \(6\times10^{-5}\) & \(4\times10^{-5}\) & \(3\times10^{-5}\) \\
Warmup steps & 198 & 267 & 382 \\
Stable steps & 8,729 & 11,750 & 16,785 \\
Cosine-decay steps & 992 & 1,335 & 1,907 \\
\bottomrule
\end{tabular*}
\end{threeparttable}
\end{table*}

Training primarily uses data parallelism, with tensor, pipeline, context, and expert parallel sizes all set to one. Random layer-set choices are distinguished from training-group replications. Within each comparison group, data order, training tokens, and optimization settings are matched. Dense execution uses Transformer Engine and FlashAttention~\citep{dao2022flashattention}.

\textbf{Early score measurement.} All profiles used for final architecture allocation are obtained at \(10\%\) of reference training. The source-side online monitor uses matched-marginal benefit measurement, with reference coverage 0.5 and low/high coverage 0.25/0.75. Let \(\mathcal L_{\ell,\mathrm{ref}=0.5}(x;\rho)\) denote loss when layer \(\ell\) uses coverage \(\rho\) and the remaining layers use reference coverage. A paired allocation score is
\begin{equation}
\begin{aligned}
a_\ell=\frac{1}{|\mathcal D|}\sum_{x\in\mathcal D}
\Bigl[&\mathcal L_{\ell,\mathrm{ref}=0.5}(x;0.25)-\mathcal L_{\ell,\mathrm{ref}=0.5}(x;0.75)\Bigr].
\end{aligned}
\label{eq:matched_marginal_score}
\end{equation}
After standard validation, the monitor replays a raw batch already consumed by that evaluation. It records current, exponential-moving-average (EMA), and cumulative scores, using EMA coefficient 0.8. The score records and frozen layer set associated with the \(10\%\) checkpoint determine the subsequent architecture.

Dense interventions apply identity rotation to unselected frequency pairs. The fixed-width MLA measurement carrier zeros unselected positional coordinates; the target MLA model constructs rotary branches at the allocated physical widths. The binary allocations use the \texttt{high\_to\_low} atom order and retain the corresponding frequency indices.

\subsection{Resource Evaluation Settings and Normalization}
\label{app:resource_protocol}

\textbf{Storage components.} Model storage comprises target-model weights and the full KV cache. We use BF16, with two bytes per weight or cache element, inference batch size one, and cache length 4,096. Tied input/output embeddings are counted once. Total parameter counts use the approximate values in the experiment configuration summary; storage is computed from these counts and the layer-wise cache dimensions.

\textbf{FLOPs convention.} A multiplication and an addition each count as one FLOP. Accounting covers layer-wise linear projections, query/key dot products, attention-weighted value products, SwiGLU projections, and vocabulary-output matrix multiplication. The forward and corresponding backward matrix multiplications together count as three forward passes. Causal attention uses \(T(T+1)/2\) valid query--key position pairs. Computation for demand-guided configurations includes \(10\%\) reference training and cumulative probing used to obtain the profile. The main tables use \(A=1\), charging the complete source cost to one target.

\textbf{Numerical profiling workload.} The main tables use a specified theoretical probing workload of two events, each processing eight global sequences of length 4,096, giving
\begin{equation}
N_{\mathrm{prof}}=2\times8\times4096=65{,}536.
\label{eq:accounting_profile_workload}
\end{equation}
The event count and global sequence count are the current accounting assumptions; event-wise global counts are unspecified in the available run summary. Other cumulative workloads are substituted linearly into the formulas in Appendix~\ref{app:early_profile_cost}.

\textbf{Resource ratios.} Storage and computation are normalized to full rotation at the same model specification:
\begin{equation}
r_S=\frac{S(\boldsymbol m)}{S(\boldsymbol m_{\mathrm{full}})},
\qquad
r_C=\frac{\mathcal C_{\mathrm{E2E}}(\boldsymbol m)}{3N_tf_t(\boldsymbol m_{\mathrm{full}})}.
\label{eq:resource_ratios}
\end{equation}
Each reference denominator is shared by configurations at the same scale, attention parameterization, and latent width. Resources for random allocations describe one training run with a given layer set.

\subsection{Parameter and Storage Counts for the Experimental Models}
\label{app:model_storage}

\textbf{Parameter counts.} Let \(d,d_{\mathrm{ff}},h,V\) denote hidden width, FFN width, head count, and vocabulary size. The experimental models satisfy \(d=64h\). Tied embeddings contain \(Vd\) parameters, and the three SwiGLU projections contain \(3dd_{\mathrm{ff}}\) matrix parameters per layer. Normalization and other auxiliary parameters are denoted by \(P_{\mathrm{aux}}\).

\textbf{Dense parameters.} Query, key, value, and output projections contain \(4d^2\) parameters per layer, giving
\begin{equation}
P_D=Vd+L\left(4d^2+3dd_{\mathrm{ff}}\right)+P_{\mathrm{aux}}.
\label{eq:dense_parameter_count}
\end{equation}
Fixed-carrier coverage configurations share this projection structure.

\textbf{MLA parameters.} Content-query and output projections contribute \(2d^2\) parameters per layer. The KV down-projection and content-key/value up-projections contribute \(3dd_c\). Rotary-query and shared positional-key projections at layer \(\ell\) contain \(d(h+1)d_{\mathrm{rope}}^{(\ell)}\) parameters. Hence,
\begin{equation}
\begin{aligned}
P_M(\boldsymbol m)
={}&Vd+L\left(2d^2+3dd_c+3dd_{\mathrm{ff}}\right)+2d(h+1)B+P_{\mathrm{aux}}.
\end{aligned}
\label{eq:mla_parameter_count}
\end{equation}
Holding the remaining specifications fixed, a rotary-budget change \(\Delta B\) changes the parameter count by \(2d(h+1)\Delta B\).

\textbf{Weights and full-layer cache.} Dense attention caches \(d\) key and \(d\) value elements per layer and token:
\begin{equation}
S_D=2P_D+2b_{\mathrm{inf}}T_{\mathrm{cache}}(2Ld).
\label{eq:dense_storage}
\end{equation}
MLA caches the KV latent and shared positional key at each layer:
\begin{equation}
S_M(\boldsymbol m)
=
2P_M(\boldsymbol m)
+2b_{\mathrm{inf}}T_{\mathrm{cache}}\left(Ld_c+2B\right).
\label{eq:mla_storage}
\end{equation}
Equal total rotary budgets therefore give equal aggregate rotary-projection and positional-key cache sizes.

\begin{table}[t!]
\centering
\setlength{\tabcolsep}{6pt}
\renewcommand{\arraystretch}{1.2}
\caption{Full- versus half-budget parameters (M) and BF16 storage (MiB; weights plus one 4K KV cache). \(r_S=S_{50}/S_{100}\).}
\label{tab:model_storage}
\begin{threeparttable}
\begin{adjustbox}{max width=\columnwidth}
\begin{tabular}{lrrrrr}
\toprule
Model & $P_{100}$ & $P_{50}$ & $S_{100}$ & $S_{50}$ & $r_S$ \\
\midrule
130M Dense & 130.2M & 130.2M & 408.34 & 408.34 & 1.000 \\
350M Dense & 349.5M & 349.5M & 1,002.62 & 1,002.62 & 1.000 \\
1B Dense & 990.8M & 990.8M & 2,561.80 & 2,561.80 & 1.000 \\ \midrule
130M MLA & 129.6M & 126.6M & 277.19 & 266.47 & 0.961 \\
350M MLA & 348.1M & 337.8M & 723.95 & 698.30 & 0.965 \\
1B MLA & 993.5M & 959.1M & 2,020.95 & 1,948.34 & 0.964 \\
\bottomrule
\end{tabular}
\end{adjustbox}
\end{threeparttable}
\end{table}

Table~\ref{tab:model_storage} combines the approximate parameter counts with Equations~\eeqref{eq:dense_storage} and~\eeqref{eq:mla_storage}. Subscripts 100 and 50 denote full and half rotary budgets; storage is in MiB, with \(1\,\mathrm{MiB}=2^{20}\) bytes. At 350M MLA, for example, half-budget allocation reduces total parameters from approximately 348.1M to 337.8M and the full-layer cache from 7,680 to 6,912 elements per token. Together, these changes reduce model storage from approximately 723.95 MiB to 698.30 MiB, corresponding to \(0.965\times\).

\subsection{Training and Configuration FLOPs for the Experimental Models}
\label{app:experiment_flops}

\textbf{Linear layers and attention.} For \(X\in\mathbb R^{T\times a}\) and \(W\in\mathbb R^{a\times b}\), forward multiplication requires \(2Tab\) FLOPs. Weight-gradient and input-gradient computations each contain a matrix multiplication with the same arithmetic count, giving \(6Tab\) for training. At each valid causal position pair, a single-head query/key dot product requires \(2d_{qk}\) FLOPs and the attention-weighted value product requires \(2d_v\). Dividing by sequence length gives the combined per-token forward cost \(h(T+1)(d_{qk}+d_v)\).

\textbf{Dense forward and training computation.} Linear projections and the FFN cost \(2T(4d^2+3dd_{\mathrm{ff}})\) per layer in the forward pass. Adding attention products across all layers and the vocabulary-output projection gives
\begin{equation}
\begin{aligned}
f_D={}&2L\left(4d^2+3dd_{\mathrm{ff}}\right)+Lh(T+1)(64+64)+2dV.
\end{aligned}
\label{eq:dense_forward_flops}
\end{equation}
Fixed-carrier allocation preserves these matrix dimensions.

\textbf{MLA forward and training computation.} Let \(\bar r=L^{-1}\sum_{\ell=1}^{L}d_{\mathrm{rope}}^{(\ell)}\) denote average rotary width, so that \(\bar r=2B/L\). The average query/key matching width is \(64+\bar r\), while value width is 64. Thus,
\begin{equation}
\begin{aligned}
f_M(\bar r)
={}&2L\left(2d^2+3dd_c+3dd_{\mathrm{ff}}+d(h+1)\bar r\right)+Lh(T+1)(128+\bar r)+2dV.
\end{aligned}
\label{eq:mla_forward_flops}
\end{equation}
Full, half, and three-quarter rotary budgets correspond to \(\bar r=64,32,48\), respectively.

\begin{table}[t!]
\centering
\setlength{\tabcolsep}{6pt}
\renewcommand{\arraystretch}{1.2}
\caption{Theoretical forward matrix FLOPs per token (GFLOPs) for Dense and full-/half-budget MLA.}
\label{tab:forward_flops}
\begin{threeparttable}
\begin{adjustbox}{max width=\columnwidth}
\begin{tabular}{lccc}
\toprule
Scale & Dense & \makecell[c]{MLA\\Full} & \makecell[c]{MLA\\Half} \\
\midrule
130M & 0.344347 & 0.384989 & 0.358114 \\
350M & 0.875112 & 0.960462 & 0.895767 \\
1B & 2.333823 & 2.515532 & 2.358617 \\
\bottomrule
\end{tabular}
\end{adjustbox}
\end{threeparttable}
\end{table}

Full target-training cost is \(3N_tf_t(\boldsymbol m)\). For half-budget MLA at the three scales, the target-training ratios are, respectively,
\begin{equation}
0.930193,\qquad 0.932642,\qquad 0.937622.
\label{eq:target_flops_ratios}
\end{equation}

\textbf{Reference-training and profiling substitutions.} The 130M source model has 20 layers, so low/high-coverage measurement performs 41 forward passes per input batch. Under the cumulative workload in Appendix~\ref{app:resource_protocol}, \((2L_s+1)N_{\mathrm{prof}}=41\times65{,}536=2{,}686{,}976\). The cumulative probing costs are
\begin{equation}
\begin{aligned}
\mathcal C_{\mathrm{prof}}^{D}=0.925251\times10^{15}\ \mathrm{FLOPs},\qquad
\mathcal C_{\mathrm{prof}}^{M}=1.034457\times10^{15}\ \mathrm{FLOPs}.
\end{aligned}
\label{eq:profile_flops_values}
\end{equation}
Early reference-training costs use the source training-token counts in Table~\ref{tab:training_hparams} and the \(10\%\) fraction. Changing the cumulative probing input gives
\begin{equation}
\mathcal C_{\mathrm{prof,new}}
=
\mathcal C_{\mathrm{prof,table}}\frac{N_{\mathrm{prof,new}}}{65{,}536}.
\label{eq:profile_workload_scaling}
\end{equation}

\textbf{End-to-end cost aggregation.} Normalize target training, early reference training, and profiling separately by full-rotation target training:
\begin{equation}
r_C=r_{\mathrm{target}}+\frac{r_{\mathrm{early}}+r_{\mathrm{prof}}}{A}.
\label{eq:e2e_ratio_components}
\end{equation}
Table~\ref{tab:e2e_components} reports the resulting decomposition with \(A=1\) and the probing workload in Appendix~\ref{app:resource_protocol}. The 130M references are same-scale; larger targets use a 130M source. Prespecified uniform configurations and given random layer sets proceed directly to target training, so their compute ratios use the target-training term. Reverse allocations generated from a demand profile include configuration acquisition in the same way.

\begin{table}[t!]
\centering
\setlength{\tabcolsep}{6pt}
\renewcommand{\arraystretch}{1.2}
\caption{Theoretical end-to-end FLOPs ratios to Full for half-budget allocation: target training, reference training, and probing.}
\label{tab:e2e_components}
\begin{threeparttable}
\begin{adjustbox}{max width=\columnwidth}
\begin{tabular}{lcccc}
\toprule
Target & Train & Reference & Probe & Total \\
\midrule
130M Dense & 1.000000 & 0.100000 & 0.000344 & \textbf{1.1003} \\
350M Dense & 1.000000 & 0.014616 & 0.000050 & \textbf{1.0147} \\
1B Dense & 1.000000 & 0.001918 & 0.000007 & \textbf{1.0019} \\
130M MLA & 0.930193 & 0.100000 & 0.000344 & \textbf{1.0305} \\
350M MLA & 0.932642 & 0.014889 & 0.000051 & \textbf{0.9476} \\
1B MLA & 0.937622 & 0.001990 & 0.000007 & \textbf{0.9396} \\
\bottomrule
\end{tabular}
\end{adjustbox}
\end{threeparttable}
\end{table}

\textbf{Compute changes in the experimental cases.} Target training for 350M MLA costs \(0.932642\times\) the full model, increasing to \(0.947582\times\) after adding early 130M reference training and profiling. The corresponding end-to-end ratio for 1B MLA is \(0.939618\times\). Dense target-matrix arithmetic remains unchanged, while the relative configuration-acquisition cost decreases with target scale.

\textbf{Configuration reuse.} When one source profile serves \(A\) targets, the 350M MLA cost becomes
\begin{equation}
r_C(A)=0.932642+\frac{0.014889+0.000051}{A}.
\label{eq:reuse_example}
\end{equation}
The same reuse rule applies to Dense and MLA using their respective source profiles.

\textbf{Comparison with target-direct profiling.} Under the same cumulative probing input and the relations in Appendix~\ref{app:configuration_cost}, early reference training for the 130M Dense source costs approximately \(0.1462\times\) and \(0.0192\times\) that of direct reference acquisition at 350M and 1B, respectively; probing costs are \(0.3292\times\) and \(0.1061\times\). The corresponding MLA early-reference ratios are \(0.1489\times\) and \(0.0199\times\), with probing ratios \(0.3354\times\) and \(0.1101\times\). These values quantify the reductions in both components of configuration acquisition obtained from a smaller model.

\subsection{Frozen Layer Sets and Layer-Selection Ablations}
\label{app:layer_ablation}

\textbf{Configuration names.} Profile50 and Profile75 use same-scale allocation scores at 50\% and 75\% rotary budgets. Transfer50 uses a within-family source profile, and Target-direct50 uses the target's early profile. These are instances of Demand-guided in the main tables. Reverse50 takes the strict complement of Profile50, while Last10 selects the final ten layers. Where multiple Random50 sets are evaluated, they represent distinct layer allocations rather than independent training replications.

\textbf{Layer-wise configurations.} Table~\ref{tab:active_layers} lists the frozen high-coverage sets. Selected layers use rotary width 64; other layers use identity rotation in Dense or zero rotary width in MLA. The same 350M MLA transferred set is used in the additional paired training runs.

\begin{table}[t!]
\centering
\setlength{\tabcolsep}{2pt}
\renewcommand{\arraystretch}{1.2}
\caption{High-coverage layer sets for each allocation; indices are zero-based and ranges are inclusive.}
\label{tab:active_layers}
\begin{threeparttable}
\begin{tabularx}{\columnwidth}{@{}>{\raggedright\arraybackslash}X>{\centering\arraybackslash}p{0.42\columnwidth}@{}}
\toprule
Model / allocation & Layers \\
\midrule
130M Dense, Profile50 & \(7,\ 10\text{--}18\) \\
130M MLA, Profile50 & \(0\text{--}2,\ 12\text{--}18\) \\
130M MLA, Profile75 & \(0\text{--}5,\ 8,\ 12\text{--}19\) \\
130M MLA, Reverse50 & \(3\text{--}11,\ 19\) \\
130M MLA, Last10 & \(10\text{--}19\) \\
350M Dense, Transfer50 & \(8,\ 12\text{--}22\) \\
350M MLA, Transfer50 & \(0\text{--}3,\ 15\text{--}22\) \\
1B Dense, Transfer50 & \(10,\ 14\text{--}26\) \\
1B MLA, Transfer50 & \(0\text{--}3,\ 17\text{--}26\) \\
1B Dense, Target-direct50 & \(6,\ 12\text{--}24\) \\
\bottomrule
\end{tabularx}
\end{threeparttable}
\end{table}

\textbf{Layer placement in 130M MLA.} Table~\ref{tab:mla130_layer_ablation} compares demand-guided, random, reversed, and contiguous late-layer allocations. The three random sets have mean loss \(3.412386\). Demand-guided allocation improves over this mean by \(0.030679\), over reverse ranking by \(0.037628\), and over the final-ten-layer configuration by \(0.087689\).

\begin{table}[t!]
\centering
\setlength{\tabcolsep}{6pt}
\renewcommand{\arraystretch}{1.2}
\caption{130M MLA allocation ablation at 50\% budget: validation loss for profile-guided, random, reversed, uniform, and last-ten-layer strategies. Random50 is averaged over three random layer sets.}
\label{tab:mla130_layer_ablation}
\begin{threeparttable}
\begin{adjustbox}{max width=2\columnwidth}
\begin{tabular}{lr}
\toprule
Configuration & \makecell[r]{Validation loss} \\
\midrule
\textbf{Profile50} & \textbf{3.381707} \\
Random50 & 3.412386 \\
Reverse50 & 3.419335 \\
Uniform50 & 3.425247 \\
Last10 & 3.469396 \\
\bottomrule
\end{tabular}
\end{adjustbox}
\end{threeparttable}
\end{table}

\textbf{Random controls in 350M MLA.} The two random sets in Table~\ref{tab:mla350_random} have mean loss \(3.062779\). Cross-scale allocation improves over the random mean by \(0.011339\).

\begin{table}[t!]
\centering
\setlength{\tabcolsep}{6pt}
\renewcommand{\arraystretch}{1.2}
\caption{350M MLA validation loss at 50\% budget: transferred profiles versus random layer sets and uniform coverage. Random50 is averaged over two random layer sets.}
\label{tab:mla350_random}
\begin{threeparttable}
\begin{adjustbox}{max width=\columnwidth}
\begin{tabular}{lr}
\toprule
Configuration & \makecell[r]{Validation loss} \\
\midrule
\textbf{Transfer50} & \textbf{3.051440} \\
Random50 & 3.062779 \\
Uniform50 & 3.067263 \\
\bottomrule
\end{tabular}
\end{adjustbox}
\end{threeparttable}
\end{table}

\subsection{Rotary Budget and Quality--Resource Trade-offs}
\label{app:budget_ablation}

We vary the number of high-coverage layers in 130M MLA along the same demand ranking, comparing half, three-quarter, and full rotary budgets. Profile50 reaches loss \(3.381707\), outperforming Profile75 at \(3.390307\) by \(0.008600\), while approaching the full-rotation loss of \(3.377954\).

\begin{table}[t!]
\centering
\setlength{\tabcolsep}{6pt}
\renewcommand{\arraystretch}{1.2}
\caption{130M MLA rotary-budget ablation: validation loss at 50\%, 75\%, and 100\% budgets, including Uniform50.}
\label{tab:mla130_budgets}
\begin{threeparttable}
\begin{adjustbox}{max width=\columnwidth}
\begin{tabular}{lcr}
\toprule
Configuration & Budget & \makecell[r]{Validation loss} \\
\midrule
Full & \(100\%\) & \textbf{3.377954} \\
Profile75 & \(75\%\) & 3.390307 \\
Profile50 & \(50\%\) & 3.381707 \\
Uniform50 & \(50\%\) & 3.425247 \\
\bottomrule
\end{tabular}
\end{adjustbox}
\end{threeparttable}
\end{table}

\begin{table}[t!]
\centering
\setlength{\tabcolsep}{6pt}
\renewcommand{\arraystretch}{1.2}
\caption{1B MLA validation loss with KV latent width 384: Full versus Uniform50.}
\label{tab:mla1b_kv384}
\begin{threeparttable}
\begin{adjustbox}{max width=\columnwidth}
\begin{tabular}{lr}
\toprule
Configuration & Validation loss \\
\midrule
\textbf{Full} & \textbf{2.731081} \\
Uniform50 & 2.745739 \\
\bottomrule
\end{tabular}
\end{adjustbox}
\end{threeparttable}
\end{table}

\begin{figure}[t!]
\centering
\includegraphics[width=0.7\linewidth]{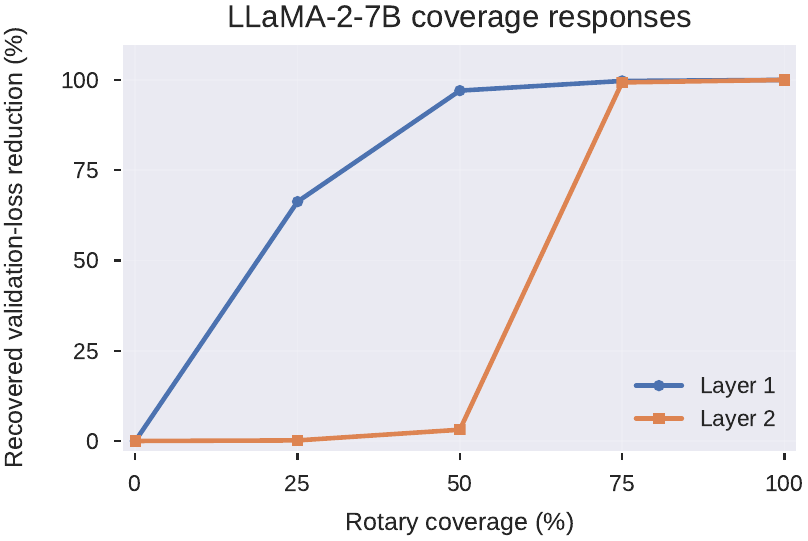}
\caption{Recovered validation-loss reduction across rotary coverage for LLaMA-2-7B layers 1 and 2 (one-based). Both curves use the same calibration texts.}
\label{fig:coverage_responses}
\end{figure}

Equations~\eeqref{eq:mla_parameter_count} and~\eeqref{eq:mla_forward_flops} make parameters and forward matrix-multiplication computation linear in average rotary width. The three-quarter budget has \(\bar r=48\), giving a target-training FLOPs ratio of \(0.965097\). Interpolating the approximate parameter counts in Table~\ref{tab:model_storage} gives combined weight and single-sequence 4K cache storage of approximately 271.83 MiB, or \(0.981\times\) the full model. Including same-scale early reference training and the specified probing workload yields an end-to-end compute ratio of \(1.0654\times\).

Demand ranking supports allocations at multiple budget levels, while the physical budget determines storage and target-training computation. In the evaluated 130M runs, the quality--budget response is non-monotone, with Profile50 offering a favorable trade-off.

\subsection{Cross-Scale and Paired Training}
\label{app:paired_transfer}

\textbf{Paired training runs.} Table~\ref{tab:mla350_repeats} reports the mean of two paired comparisons for 350M MLA. Transfer50 reduces validation loss by \(0.011348\) on average relative to Uniform50. The reported \(\Delta\) is computed within each matched training group before averaging.

\begin{table}[t!]
\centering
\setlength{\tabcolsep}{6pt}
\renewcommand{\arraystretch}{1.2}
\caption{350M MLA validation loss averaged over two paired training runs. \(\Delta\) is the mean within-pair difference (Uniform50 minus Transfer50).}
\label{tab:mla350_repeats}
\begin{threeparttable}
\begin{adjustbox}{max width=\columnwidth}
\begin{tabular}{lccc}
\toprule
Statistic & Uniform50 & Transfer50 & \makecell[c]{$\Delta$ Validation loss} \\
\midrule
Mean & 3.060510 & \textbf{3.049162} & \textbf{0.011348} \\
\bottomrule
\end{tabular}
\end{adjustbox}
\end{threeparttable}
\end{table}

\textbf{Target-direct profiling.} At 1B Dense, we also construct an allocation using the target model's own profile at \(10\%\) of reference training and compare it with uniform coverage from the same training group. Target-direct allocation improves loss by \(0.013850\). Its high-coverage layers are \(6,12\text{--}24\), sharing eleven layers with the transferred set and giving a set Jaccard score of \(11/17\). Layers exclusive to source transfer are \(10,25,26\), while those exclusive to target-direct allocation are \(6,12,13\).

\begin{table}[t!]
\centering
\setlength{\tabcolsep}{6pt}
\renewcommand{\arraystretch}{1.2}
\caption{1B Dense validation loss with target-direct profiling versus Uniform50 at 50\% budget.}
\label{tab:dense1b_direct}
\begin{threeparttable}
\begin{adjustbox}{max width=\columnwidth}
\begin{tabular}{lr}
\toprule
Configuration & Validation loss \\
\midrule
Uniform50 & 2.735777 \\
\textbf{Target-direct50} & \textbf{2.721927} \\
\bottomrule
\end{tabular}
\end{adjustbox}
\end{threeparttable}
\end{table}

Source transfer and target-direct measurement both supply architecture-selection signals: the former reuses small-model structure, and the latter uses the target's early layer-wise distribution. Their acquisition costs follow the separate accounting in Appendices~\ref{app:configuration_cost} and~\ref{app:experiment_flops}.

\subsection{Additional Experiments with a Narrower KV Latent}
\label{app:kv384}

This group uses a 1B MLA model with KV latent width 384. The remaining structure has 28 layers, hidden width 1,536, FFN width 4,096, and 24 heads, with the 1B training settings in Table~\ref{tab:training_hparams}.

Uniformly reducing rotary width increases loss by \(0.014658\). The full-layer logical cache decreases from \(28(384+64)=12{,}544\) to \(28(384+32)=11{,}648\) elements per token. Substituting this latent width into Equations~\eeqref{eq:mla_parameter_count}, \eeqref{eq:mla_storage}, and~\eeqref{eq:mla_forward_flops} gives full-model storage and training computation. This experiment complements the rotary-budget--quality analysis with a narrower latent specification.

\subsection{Representative Coverage Responses}
\label{app:coverage_responses}

Figure~\ref{fig:coverage_responses} compares LLaMA-2-7B layers 1 and 2, as discussed in Section~\ref{sec:coverage_calibration}. The ordinate is the recovered utility \(\widehat G_\ell(0,\rho)\), divided by each layer's phase-off loss on the same 100 calibration samples. At half coverage, layer 1 recovers approximately 97.0\% of its full-rotation utility, whereas layer 2 recovers 3.1\%; the latter rises to 99.3\% at three-quarter coverage.

\end{document}